%% file: ms.tex
\pdfoutput=1
\documentclass{article}

\usepackage{microtype}
\usepackage{graphicx}
\usepackage{booktabs} % for professional tables
\usepackage{amsfonts}       % blackboard math symbols
\usepackage{nicefrac}       % compact symbols for 1/2, etc.
\usepackage{wrapfig}
\usepackage{grffile}
\usepackage{sidecap}
\usepackage{amsmath}
\usepackage{subcaption}
\usepackage[export]{adjustbox}
\usepackage{xcolor}
\usepackage{enumitem}
\usepackage{xcolor}
\usepackage{placeins}
\usepackage[accepted,nohyperref]{icml2021}
\usepackage[utf8]{inputenc} % allow utf-8 input
\usepackage[T1]{fontenc}    % use 8-bit T1 fonts
\usepackage{hyperref}       % hyperlinks
\usepackage{url}            % simple URL typesetting
\usepackage{algorithm}
\usepackage{algorithmic}

\usepackage{hyperref}

\input{defs}
\begin{document}

\twocolumn[
\icmltitle{Generative Minimization Networks: Training GANs Without Competition}



\icmlsetsymbol{equal}{*}

\begin{icmlauthorlist}
\icmlauthor{Paulina Grnarova}{to}
\icmlauthor{Yannic Kilcher}{to}
\icmlauthor{Kfir Y. Levy}{goo}
\icmlauthor{Aurelien Lucchi}{to}
\icmlauthor{Thomas Hofmann}{to}

\end{icmlauthorlist}

\icmlaffiliation{to}{ETH Zurich}
\icmlaffiliation{goo}{Technion-Israel Institute of Technology}

\icmlcorrespondingauthor{Paulina Grnarova}{paulina.grnarova@inf.ethz.ch}

\icmlkeywords{}

\vskip 0.3in
]



\printAffiliationsAndNotice{}  % leave blank if no need to mention equal contribution

\begin{abstract}
Many applications in machine learning can be framed as minimization problems and solved efficiently using gradient-based techniques. However, recent applications of generative models, particularly GANs, have triggered interest for solving min-max games for which standard optimization techniques are often not suitable. Among known problems experienced by practitioners are the lack of convergence guarantees or convergence to a non-optimum cycle. At the heart of these problems is the min-max structure of the GAN objective which creates non-trivial dependencies between the players.
We propose to address this problem by optimizing a different objective that circumvents the min-max structure using the notion of duality gap from game theory. We provide novel convergence guarantees on this objective and demonstrate why the obtained limit point solves the problem better than known techniques.
\end{abstract}

\FloatBarrier

\input{01_introduction}
\input{02_background}

\input{03_theory}
\input{04_related_work}
\input{05_experiments}

\input{06_conclusion}

\clearpage
\newpage
\bibliography{main}
\bibliographystyle{abbrvnat}
%\setcitestyle{authoryear,open={(},close={)}}
%\bibliographystyle{apalike}

\clearpage
\newpage
\input{07_appendix}

\end{document}

%% file: defs.tex
\renewcommand{\u}{{\bf u}}
\renewcommand{\v}{{\bf v}}

\newcommand{\x}{{\bf x}}
\newcommand{\y}{{\bf y}}

\def\Am{{\bf A}}

\def\Im{{\bf I}}

\newcommand{\E}{{\mathbf E}}

\usepackage{amsthm}
\newtheorem{Def}{Definition}
\newtheorem{Ass}{Assumption}
\newtheorem{Thm}{Theorem}
\newtheorem{Lem}[Thm]{Lemma}

\def\mP{{\mathcal P}}
\def\reals{{\mathcal R}}

\newcommand{\Xcal}{\mathcal{X}}

\newcommand{\Zcal}{\mathcal{Z}}

\newcommand{\ignore}[1]{}

\def\reals{{\mathbb R}}
\def\mP{{\mathcal P}}

\def\bold0{\mathbf{0}}

\def\x{\mathbf{x}}
\def\y{\mathbf{y}}

\def\v{\mathbf{v}}

\newcommand{\DG}{{\rm DG}}

%% file: 01_introduction.tex
\section{Introduction}

\begin{figure}[t]
\begin{subfigure}{.25\textwidth}
  \centering
  \includegraphics[scale=0.2]{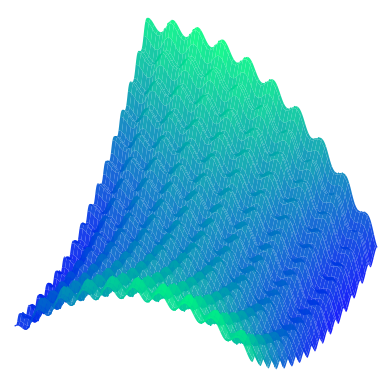}
  \caption{Minimax Objective}
\end{subfigure}%
\begin{subfigure}{.25\textwidth}
  \centering
\includegraphics[scale=0.2]{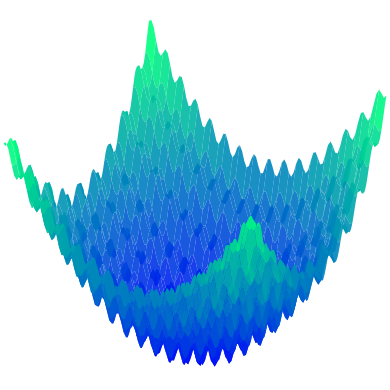} 
  \caption{Duality Gap Objective}
\end{subfigure}
\caption{\label{fig:motivation} Optimization landscapes: \textbf{(a:)} Minimax objective function with many local Nash equilibria. The goal is to find a (local) NE. \textbf{(b:)} Duality Gap (DG) objective function. The (local) NE are turned into (local) minima and the minimax problem is transformed into an optimization of the DG wrt. both players.}
\end{figure}
\vspace{-0.1cm}
Many of the core applications encountered in the field of machine learning are framed as minimizing a differentiable objective. Often, the method of choice to optimize such a function is a gradient-descent (or a related) method which consists in simply stepping in the negative gradient direction. This algorithm has become a defacto due to its simplicity and the existence of convergence guarantees, even for  functions that are not necessarily convex.

On the other hand, the rise of generative models, particularly the GAN framework from~\citet{goodfellow2014generative}, has triggered significant interest in the machine learning community for optimizing minimax objectives of the form
\begin{equation}
\min_{\u \in U} \max_{\v \in V} M(\u,\v)
\label{eq:minmax}
\end{equation}
The latter objective is a considerably more challenging problem to solve in a general setting as it requires optimizing multiple objectives jointly. The typical notion of optimality used for such games is the concept of Nash equilibrium (NE) where no player can improve its objective by unilaterally changing their strategy. As discussed in~\cite{jin2019local}, finding a global NE is NP-hard in a general setting where the minimax objective is not convex-concave. Instead, one has to settle for a local NE or a different type of local optimality.

In practice, minimax problems are still solved using gradient-based algorithms, especially gradient descent-ascent (GDA) that simply alternates between a gradient descent step for $\x$ and a gradient ascent step for $\y$. There are (at least) three problems with GDA in GANs and games in general: (i) the potential existence of cycles implies there are no convergence guarantees~\citep{mescheder2018training}, (ii) even when gradient descent converges, the rate may be too slow in practice because the recurrent dynamics require extremely small learning rates ~\citep{mescheder2018training, gidel2018negative} and (iii) since there is no single objective, there is no way to measure progress ~\citep{balduzzi2018mechanics}. Recently, ~\cite{grnarova2019domain} addressed (iii) by proposing the duality gap (DG) as a metric for evaluating GANs that naturally arises from the game-theoretic aspect. The authors show the metric correlates highly with the performance and quality of GANs. Concretely, the evolution of the DG is shown to track convergence of the algorithm to an optimum. In this work, we go further and propose to use DG as a training objective in order to address (i) and (ii). In fact, we argue that the duality gap is \textbf{the objective} that should be minimized for training a GAN.

% Intuitively, there are several reasons for this. First, DG as a training objective turns the difficult minimax problem into a simpler minimization problem. This is illustrated in Fig.~\ref{fig:motivation} that shows the landscape of a minimax objective for a nonconvex-nonconcave game and the landscape of the corresponding Duality Gap function. Concretely, a pure NE (if it exists) turns into a global minimum and the new goal becomes to converge to a local or global minimum as in standard optimization, where both players are jointly minimizing the same objective. This is not just more intuitive, but also allows us to reuse tools from the literature (\pg{examples?}). Additionally, as we show later in the experimental section, when optimizing the DG, some of the points that are undesired solutions, yet are points where GDA and alternatives converge to, turn into maxima or saddle points, thus become easier to avoid. \pg{It would be great to really make a point of this if we can: like undesired points have a high DG thus are easier to avoid, while desired points become more stable}. Lastly, moving to the standard setting of minimization allows us to have access to train and validation curves which enables practitioners to track progress and avoid overfitting.
% \pg{we used to say (local) NE turn into (local) min. do we still say that or avoid it as it might not be true?}

Intuitively, there are several reasons for this. First, instead of looking for a Nash equilibrium of the minimax problem, one can interpret the optimization of the DG as a minimization problem whose solution obeys the typical criticality condition used in optimization.

This is illustrated in Fig.~\ref{fig:motivation} that shows the landscape of a minimax objective for a nonconvex-nonconcave game and the landscape of the corresponding Duality Gap function.
Concretely, a pure NE (if it exists) turns into a global minimum and the new goal becomes to converge to a local or global minimum as in standard optimization, where both players are jointly minimizing the same objective.
This has multiple advantages such as: i) one can simply rely on minimization optimization methods for non-convex functions, for which stronger convergence guarantees exist~\citep{jain2017non, dauphin2014identifying,nesterov2006cubic}, and ii) it suppresses the competition aspect between players, which is responsible for cyclic behaviors and the lack of stability near stationary points. We support the latter claim in the experimental section, where we demonstrate that, GDA might converge to points that are undesired solutions, which are turned into maxima or saddle points when optimizing the duality gap, and thus become easy to avoid. Lastly, moving to the standard setting of optimization allows us to have access to train and validation curves which enables practitioners to track progress and avoid overfitting.

In order to verify the validity of our approach, we derive a convergence rate under a similar set of assumptions as the ones typically found in the GAN literature for min-max objectives~\citep{goodfellow2014generative, nowozin2016f}~\footnote{Note that these prior works assume convexity in the function space. Since one can not typically optimize in the function space, we naturally place our assumptions on the parameter space.}. Concretely, we prove the limit point of our algorithm is guaranteed to converge to a point of zero divergence and we also derive a rate of convergence. We then check the validity of our theoretical results on a wide range of (assumption-free) practical problems. 
Finally, we empirically demonstrate that changing the nature of the problem from an adversarial game theoretic to an optimization setting yields a training algorithm that is more robust to the choice of hyperparameters.
For instance, it avoids having to choose different learning rates for each player, which is known to be an effective technique for  GANs~\citep{heusel2017gans}.

In summary, we make the following contributions:
\begin{itemize}[noitemsep,topsep=0pt]
    \item We propose a new objective function for training GANs without relying on a complex minimax structure and we derive a rate of convergence.
    \item We prove that adaptive optimization methods are suitable for training our objective and that they exploit certain properties of the objective that allow for faster rates of convergence.
    \item We propose a simple practical algorithm that turns the game theoretic formulation of GAN (as well as WGAN, BEGAN etc.) into an optimization problem.
    \item We validate our approach empirically on several toy and real datasets and show it exhibits desirable convergence and stability properties, while attaining improved sample fidelity. 
\end{itemize}

% %%%%%%%%%%%%%%%%%%%

%% file: 02_background.tex
\vspace{-2.1mm}
\section{Duality Gap as a Training Objective}
\vspace{-0.2cm}
We now introduce some key game-theoretic concepts that will be necessary to present our algorithm.

A zero-sum game consists of two players $P_1$ and $P_2$ who choose a decision from their respective decision sets $U$ and $ V$. A game objective $M: U \times  V \mapsto \reals$ defines the utilities of the players. 
Concretely, upon choosing a pure strategy $(\u,\v)\in  U \times  V$ the utility of $P_1$ is $-M(\u,\v)$, while the utility of $P_2$ is  $M(\u,\v)$. The goal of either $P_1$/$P_2$ is to maximize their worst case utilities; thus,
\vspace{-0.2cm}
\begin{align}
\label{eq:MinmaxGame}
\min_{\u\in U}\max_{\v\in  V}M(\u,\v) \quad \textbf{(Goal of $P_1$)},  \\
\nonumber 
\max_{\v\in V}\min_{\u\in  U}M(\u,\v)\quad \textbf{(Goal of $P_2$)}
\end{align} 
The above formulation raises the question of whether there exists a solution  $(\u^*,\v^*)$ to which both players may jointly converge.
The latter only occurs if there exists  $(\u^*,\v^*)$ such that  neither $P_1$ nor $P_2$ may increase their utility by unilateral deviation.  Such a solution is a \emph{pure equilibrium}, and is formally defined as follows,
$$
\max_{\v\in V}M(\u^*,\v) = \min_{\u\in U}M(\u,\v^*)~ ~~\textbf{(Pure Equilibrium).}
$$
% While a pure equilibrium does not always exist, 
% the seminal work of~\cite{nash1950equilibrium} shows that an extended notion of equilibrium always does. Specifically, there always exists a distribution $\D_1$ over elements of $ U$, and a distribution $\D_2$ over elements of $ V$, such that the following holds,
% $$
% \max_{\v\in V}\E_{\u\sim \D_1}M(\u,\v) = \min_{\u\in U}\E_{\v\sim \D_2}M(\u,\v)~ \\ ~~\textbf{(MNE).}
% $$
% Such a solution is called a \emph{Mixed Nash Equilibrium (MNE)}.
This notion of equilibrium gives rise to the natural performance measure of a given pure strategy:
\begin{Def}[Duality Gap]
% Let $\D_1$ and $\D_2$ be fixed distributions over elements from $ U$ and $ V$ respectively. Then the  duality gap $\DG$ of $(\D_1,\D_2)$ is defined as follows,
% \begin{align}
% \rm{DG}:=
%  \max_{\v\in V}\E_{\u\sim \D_1}M(\u,\v)
%  -\min_{\u\in U}\E_{\v\sim \D_2}M(\u,\v).
% \end{align}
For a given \emph{pure strategy} $(\u,\v)\in U\times  V$ we define,
\begin{align}
\rm{DG}(\u,\v):=
 \max_{\v'\in V}M(\u,\v')~
 -\min_{\u'\in U}M(\u',\v)~.
\end{align}
\end{Def}

%% file: 03_theory.tex
\vspace{-2mm}
\subsection{Properties of the Duality Gap}
\vspace{-1mm}

In the context of GANs,~\cite{grnarova2019domain} showed as long as $G$ is not equal to the true distribution then the duality gap is always positive. In particular, the duality gap is at least as large as the Jensen-Shannon divergence between true and fake distributions (which is always non-negative)\footnote{Similarly DG is larger than other divergences for other minimax GAN formulations, such as WGAN~\citep{grnarova2019domain}}. Furthermore, if $G$ outputs the true distribution, then there exists a discriminator such that the duality gap is zero.

Building on the property of the duality gap as an upper bound on the divergence between the data and the model distributions, it appears intuitive that one could train a generative model by optimizing the duality gap to act as a surrogate objective function.
Doing so has the advantage of reducing the training objective to a standard optimization problem, therefore bypassing the game theoretic formulation of the original GAN problem.
This however raises several questions about the practicality of such an approach, its rate of convergence and its empirical performance compared to existing approaches. We set out to answer these questions next.
Formally, the problem we consider is: 
\begin{align}
(\u^*, \v^*) = \min_{\u, \v} \left[ \DG(\u,\v) \right] %:= \max_{\v' \in V} M(\u, \v') - \min_{\u' \in U} M(\u', \v) \right]
\label{eq:DG}
\end{align}
There are several types of convergence guarantees that are desirable and commonly found in the GAN literature, including (i) stability around stationary points as in~\cite{mescheder2018training}, and (ii) global convergence guarantees as in~\cite{goodfellow2014generative, nowozin2016f}.
We start by deriving the second types of guarantees in Section~\ref{sec:theory}. We then present a practical algorithm whose stability is discussed in Section~\ref{sec:algorithm}.
\vspace{-1mm}

\subsection{Theoretical guarantees}
\label{sec:theory}
\vspace{-1mm}

In this section we analyze the performance of first-order methods for optimizing the DG objective (Eq.~\ref{eq:DG}). We prove that standard adaptive methods (i.e., AdaGrad) yield faster convergence compared to standard primal-dual methods such as gradient descent-ascent and extragradient, which are commonly used for GANs.

We develop our analysis under an assumption known as \emph{realizability}, which can be enforced during training using techniques introduced in prior work, e.g.~\cite{dumoulin2016adversarially}.
We prove that, in a stochastic setting, AdaGrad~\citep{duchi2011adaptive} converges to the optimum of Eq.~\eqref{eq:DG} at a rate of $O(1/T)$, where $T$ is the number of gradient updates (proportional to the number of stochastic samples). This directly translates to an $O(1/T)$-approximate pure equilibrium for the original minimax problem, which substantially improves over the rate of $O(1/\sqrt{T})$ for the stochastic minimax setting~\citep{juditsky2011first}.
Even under this realizability assumption, we are not aware of a similar result for minimax primal-dual algorithms such as stochastic gradient descent-ascent and extragradient. 

\vspace{-2mm}
\paragraph{Minimax problems \& Realizability.}
Formally, we consider stochastic minimax problems:
\vspace{-2mm}
\begin{align} \label{eq:MinimaxStochastic}
\min_{\u\in U}\max_{\v\in V} \left[M(\u,\v) : = \E_{z\sim\mP}M(\u,\v;z) \right]
\end{align}
We make the following assumption,
\begin{Ass}[Realizability] \label{assum:realizability}
There exits a pure equilibrium $(\u^*,\v^*)$ of $M$ such that $(\u^*,\v^*)$ is also the pure equilibrium of $M(\cdot,\cdot;z)$ for any $z\in \text{supp}(\mP)$, where  \text{supp}($\mP$) is the support of $\mP$.
\end{Ass}
Let $\DG(\cdot,\cdot;z)$ be the duality gap defined as in Eq.~\eqref{eq:DG} for $M(\cdot,\cdot;z)$. Then realizability implies that the original minimax problem is equivalent to solving the following stochastic minimization problem,
\begin{align} \label{eq:DG_Stochastic}
\min_{\u\in U,\v\in V} \left[ \DG(\u,\v): =  \E_{z\sim \mP}\DG(\u,\v;z) \right]~.
\end{align}
\textbf{Remark:} In Assumption~\ref{assum:realizability} we assume \emph{perfect}  realizability which might be too restrictive. 
In the appendix we extend our discussion to the case
where  realizability  only holds \emph{approximately}; in this case we show an approximate equivalence between the formulations of Equations~\eqref{eq:MinimaxStochastic} and~\eqref{eq:DG_Stochastic}.

%In the appendix extend the above relation to the case where the realizability assumption only holds approximately. In the latter case, we show an approximate equivalence between the formulations of Equations ~\eqref{eq:MinimaxStochastic} and ~\eqref{eq:DG_Stochastic}.
\vspace{-1mm}
\textbf{Realizability in the context of GANs.}
Training GANs  is equivalent to solving a stochastic minimax problem with the following objective,
\begin{align}
M(\u,\v) &= \E_{(z_1,z_2)\sim \mP} [ M(\u,\v;z_1,z_2): =   \log D_{\v}(z_1) \nonumber \\
& \qquad\qquad\qquad + \log(1-D_{\v}(G_{\u}(z_2)))~],
\end{align}
\vspace{-1mm}
where $\u$ and $\v$ are the respective weights of the generator and discriminator, and the source of randomization is a vector $z = (z_1,z_2)$ where $z_1$ corresponds to random samples associated to the true data, and $z_2$ is the noise that is injected in the generator. 
Commonly, $z_1$ is distributed uniformly over the true samples, and $z_2$ is a standard random normal vector. The joint distribution $\mP$ over $(z_1,z_2)$ can be any  distribution such that marginal distribution of $z_1$ (respectively $z_2$) is uniform (respectively Normal)  \footnote{ This is since $M(\u,\v;z_1,z_2)$ is separated into two additive terms, each depends on either $z_1$ or $z_2$.}. 
While the most natural choice for $\mP$ is to sample $z_1,z_2$ independently, we shall allow them to depend on each other.

Thus, in the context of GANs,  realizability  implies that there exists a pair of generator-discriminator $(\u^*,\v^*)$ such that for any sample $(z_1,z_2)\in \text{supp}(\mP)$ then $(\u^*,\v^*)$ is also a pure equilibrium of $M(\cdot,\cdot,z_1,z_2)$.
A natural case where realizability applies is when the following assumption holds,
\begin{Ass}
[``Perfect" generator]
There exist $\u^*\in U$ such $G_{\u^*}(z_2) = z_1$ for any $(z_1,z_2)\sim \mP$.
\end{Ass}

\vspace{-1mm}
This means that  there exists a ``perfect" generator $\u^*$ that for any pair of true sample $z_1$ and matching noise injection $z_2$, then $G_{\\u^*}$ perfectly reproduces $z_1$ from $z_2$. In this case the optimal discriminator $\v^*$ outputs $D_{\v^*}(z_1) = D_{\v^*}(G_{\u^*}(z_2))=1/2~, \forall (z_1,z_2)\sim\mP$. As seen in Fig.~\ref{fig:mog}, the above assumption holds for well trained GANs in practice and $D_{\v^*}(\cdot)$ is quite concentrated around $1/2$ for both real and generated samples.

The question of sampling from the joint distribution of \emph{matching} (true data, noise injection) pairs $(z_1,z_2)$, was recently addressed in~\cite{dumoulin2016adversarially}, where the authors devise a novel GAN architecture that produces these pairs throughout the training. In our experiments we use their architecture to sample matching  $(z_1,z_2)$ dependently. Additionally, \cite{dumoulin2016adversarially} also show that our ``perfect generator`` assumption holds approximately well in practice (see figures 2-4 therein).

\vspace{-7pt}
\paragraph{Fast convergence for DG optimization under realizability.} We have seen that realizability  implies we can turn  a stochastic minimax problem as in Eq.~\eqref{eq:MinimaxStochastic}, into a stochastic minimization problem as in Eq.~\eqref{eq:DG_Stochastic}. Notably, this enables the use of standard optimization algorithms to train a generative model.
It is well known that under realizability, SGD achieves a convergence rate of $O(1/T)$ for stochastic convex optimization problems \citep{moulines2011non,needell2014stochastic}, which improves over the standard $O(1/\sqrt{T})$ rate. 
However, in the realizable case, SGD requires prior knowledge about the smoothness of the problem which is usually unknown in advance.
Moreover, in practice one often uses adaptive methods such as AdaGrad \citep{duchi2011adaptive} and Adam \citep{kingma2015adam} for training generative models. Next we demonstrate a simple version of Adagrad can in fact achieve the fast rate of 
$O(1/T)$, without knowing the smoothness constant of the objective, $L$.

Our goal is to solve stochastic optimization problems as in Eq.~\eqref{eq:MinimaxStochastic}, and we assume that each $M(\cdot,\cdot;z)$ is convex-concave. Under realizability, this translates 
to solving a stochastic \emph{convex} optimization problem as in Eq.~\eqref{eq:DG_Stochastic}.
To solve this we consider the following  update,
\begin{align}
(\u_{t+1},\v_{t+1}) = (\u_{t},\v_{t}) - \eta_t g_t& \; \eta_t = D/\left( \sum\nolimits_{\tau=1}^t  \|g_\tau\|^2 \right)^\frac{1}{2}~ \nonumber \\ \textbf{(simplified AdaGrad)} 
\label{eq:adagrad}
\end{align}
where $g_t$ is an unbiased estimate of $(\nabla_\u \DG(\u_t,\v_t),\nabla_\v \DG(\u_t,\v_t))$.  Note that $\eta_t$ does not depend on the smoothness constant $L$.
Next we show that this AdaGrad variant ensures an  $O(1/T)$ rate under realizability. Let $\nabla_{(\u,\v)} \DG(\cdot,\cdot;z)$ be the concatenation of the gradients of $\DG$ w.r.t.~$\u$ and $\v$.
\begin{Lem}\label{lem:AdaGrad}
Consider a stochastic optimization problem in the form of Eq.~\eqref{eq:DG_Stochastic}. Further assume that $\DG(\cdot;z)$ is $L$-smooth~\footnote{$L$-smoothness means that the gradients of $\DG(\cdot;z)$ are $L$-Lipschitz, i.e. $\|\nabla_{(\u,\v)} \DG(\u_1,\v_1;z) -\nabla_{(\u,\v)} \DG(\u_2,\v_2;z)\| \leq L (\|\u_1-\u_2\|+\|\v_1-\v_2\|)~,\forall (\u_1,\v_1),(\u_2,\v_2)\in U\times V, z\in \text{supp}(\mP)$.} and convex  $\forall z\in \text{supp}(\mP)$, and the realizability assumption holds, i.e., there exists $(\u^*,\v^*)\in U\times V$ such that $ \nabla_{(\u,\v)} \DG(\u^*,\v^*;z) = 0~,~\forall z\in \text{supp}(\mP)$. Then applying AdaGrad to this objective ensures an $O(L D^2/T)$ convergence rate where $D$ is the diameter of $U\times V$. 
\end{Lem}
\vspace{-1mm}
A direct consequence of Lemma~\ref{lem:AdaGrad} together with Theorem~1 in~\cite{grnarova2019domain} is that our algorithm is guaranteed to converge to a point of zero divergence between the true and generated distributions.

%%%%%%%%%%%%%%%%%%%%

\vspace{-2mm}
\subsection{Algorithm}
\label{sec:algorithm}
\vspace{-2mm}
In the previous section we established convergence guarantees for DG that are on par with the ones found in the GAN literature, e.g.~\citep{goodfellow2014generative, nowozin2016f}. However, computing the DG exactly requires finding the worst case max/min player or discriminator/generator, 
$\v_{\rm{worst}}=\max_{\v\in V} M(\u,\v)$, and $\u_{\rm{worst}}=\min_{\u\in U} M(\u,\v)$ in Eq.~\ref{eq:DG}. In practice, this is not always feasible or computationally efficient, and we therefore estimate $\v_{\rm{worst}}$ and $\u_{\rm{worst}}$ using gradient-based optimization with a finite number of steps denoted by $k$ (see Alg.~\ref{alg:mindg}). Similar procedures are used in the GAN literature, where each player only solves its own loss approximately, potentially using more steps as suggested in~\cite{goodfellow2014generative, arjovsky2017wasserstein} or using an unrolling procedure~\citep{metz2016unrolled}. To speed up the optimization, we initialize the networks using the parameters of the adversary at the particular step being evaluated. As discussed in ~\cite{grnarova2019domain} this initialization scheme does not only speed up optimization, but also ensures that the practical DG is non-negative as well. Next we demonstrate that the approximation of the DG enjoys the desired properties and still leads to learning good generative models.
\begin{algorithm}[h]
\caption{Minimizing approx-DG}
\label{alg:mindg}
\begin{algorithmic}
\STATE \textbf{Input}: \#steps $T$, game objective $M(\u,\v)$, $k(=10)$ \\
\FOR{$t=1 \ldots T$ }
\STATE \text{set} $\u_{w_0}=\u_{t}$  and $\v_{w_0}=\v_{t}$
%\STATE compute $\u_{worst}$ and $\v_{worst}$.\\
{\small
\FOR{$i=1 \ldots k$ } 
\item[] \text{\%\% Computing  $\u_{\rm{worst}}$ and $\v_{\rm{worst}}$\%\%}
%\text{for i=1 $\ldots$ k do:}
\begin{align*}
\u_{w_{i}} \gets\u_{w_{i-1}}- \gamma_i\cdot\nabla_{\u_{w}}M(\u_{w_{i}}, \v_{t})
\\
\v_{w_{i}} \gets\v_{w_{i-1}}+  \gamma_i\cdot\nabla_{\v_{w}}M(\u_{t}, \v_{w_{i}})
\end{align*}
\ENDFOR
}
%\STATE{Set:}$\u_{\rm{worst}} = \u_{w_{k}}$ and $\v_{\rm{worst}}= \v_{w_{k}}$
\STATE{calculate:} $DG(\u_{t}, \v_{t}) = M(\u_t,\v_{w_{k}}) - M(\u_{w_{k}},\v_t)$ 
\\
\STATE 
\begin{align*}
  \u_{t+1} \gets\u_{t}- \eta_t\cdot\nabla_{\u}DG(\u_{t}, \v_{t}) \\
\v_{t+1} \gets\v_{t}- \eta_t\cdot\nabla_{\v}DG(\u_{t}, \v_{t})  
\end{align*}
\ENDFOR
\end{algorithmic}
\end{algorithm}
The effect of $k$ for approximating the DG has been analyzed in~\cite{grnarova2019domain}. In summary, even for small values of $k$, the DG accurately reflects the (true) DG and the quality of the generative model. This is further supported by~\cite{schafer2019implicit} that shows it takes a few update steps for the discriminator to pick up on the quality of the generator. 

\textbf{Landscape} As mentioned, the DG converts the setting from adversarial minimax game with the goal to find a (local) NE to an optimization setting. In Fig.~\ref{fig:approximate_landscape} we show the minimax landscape of a game with 3 NE and one bad stationary point~\citep{mazumdar2019finding} and its transformation when the objective changes to DG. We show both the true (theoretical) DG, as well as approximations for various values of $k$, and demonstrate they are able to closely approximate the true landscape, especially around the critical points. In particular, the DG value is the lowest for the NE, and is high for the bad stationary point across all approximations.

\begin{figure*}
\centering
\begin{subfigure}{.21\textwidth}
  \centering
  \includegraphics[width=1\textwidth]{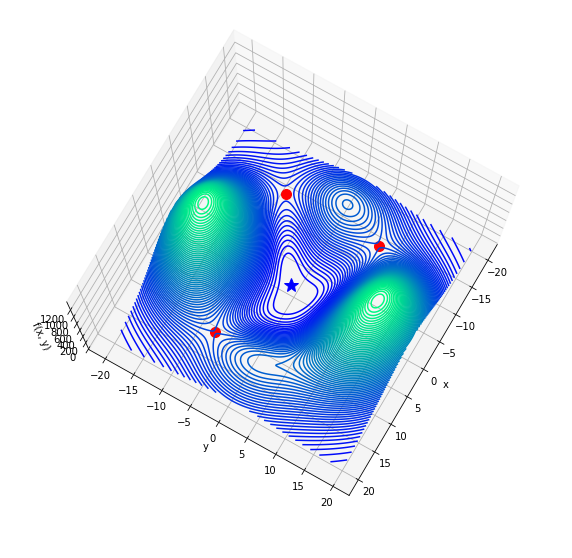}
  \caption{Minimax}
  \label{fig:sub11}
\end{subfigure}%
\begin{subfigure}{.21\textwidth}
  \centering
  \includegraphics[width=1\textwidth]{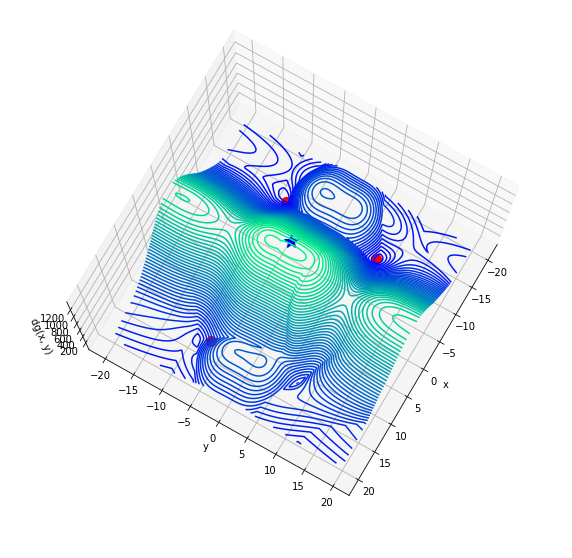}
  \caption{DG k=10}
  \label{fig:sub12}
\end{subfigure}%
\begin{subfigure}{.21\textwidth}
  \centering
  \includegraphics[width=1\textwidth]{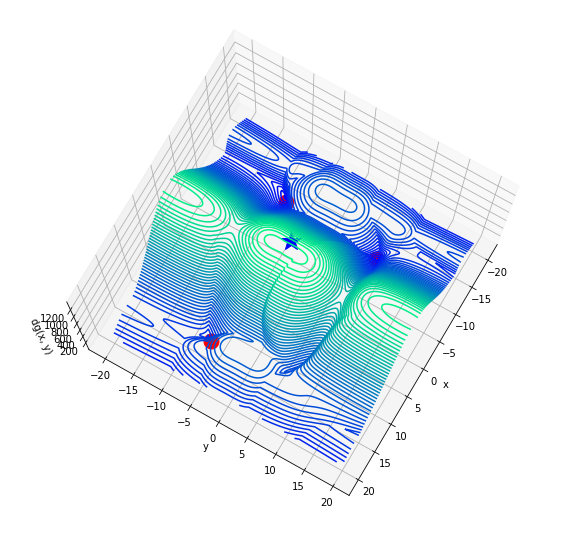}
  \caption{DG k=25}
  \label{fig:sub13}
\end{subfigure}%
\begin{subfigure}{.21\textwidth}
  \centering
  \includegraphics[width=1\textwidth]{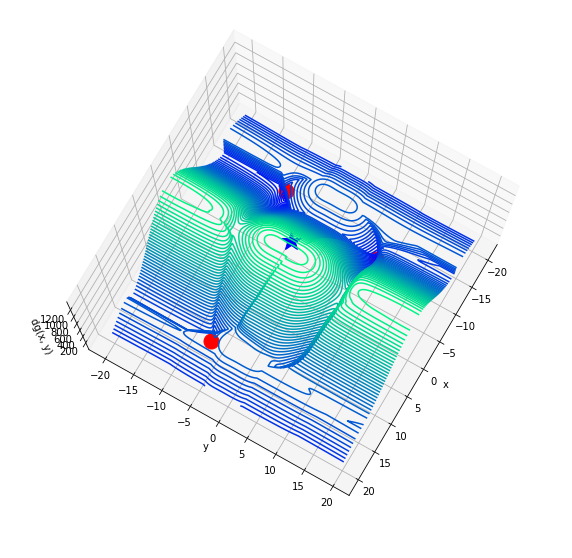}
  \caption{DG k=50}
  \label{fig:sub14}
\end{subfigure}%
\begin{subfigure}{.2\textwidth}
  \centering
  \includegraphics[width=1\textwidth]{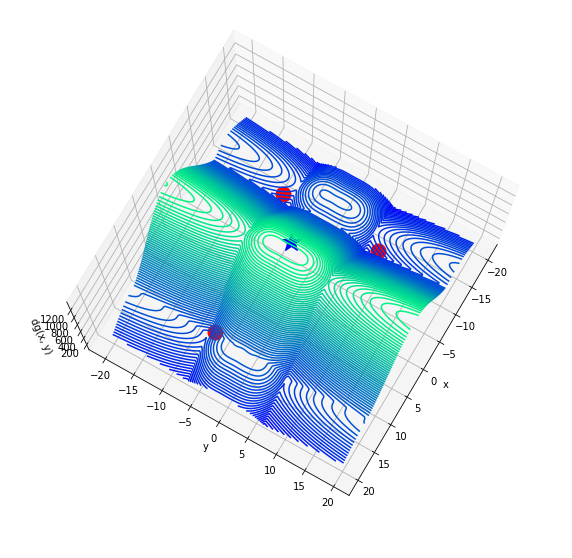}
  \caption{DG (true)}
\end{subfigure}
\caption{Landscape of a game with 3 local NE (red dots) and one bad stationary point (blue star). a) is the minmax setting with the goal to converge to a red dot. e) is the true DG and b-d) its approximations for different k. The red points have lowest DG values (b-e)}
\label{fig:approximate_landscape}
\vspace{-5mm}
\end{figure*}
\vspace{-0.1cm}
\textbf{Convergence} In Fig.~\ref{fig:ccandncc_toy} we empirically demonstrate that using DG as an objective in a convex-concave setting yields convergence to the solution of the game, both for the true and the approximate DG (similar examples for nonconvex-concave setting can be found in the appendix). Since this example is commonly analyzed in the literature, we follow by a stability analysis of the practical algorithm.

\begin{figure}
\centering
\begin{subfigure}{.2\textwidth}
  \centering
  \includegraphics[width=1\textwidth]{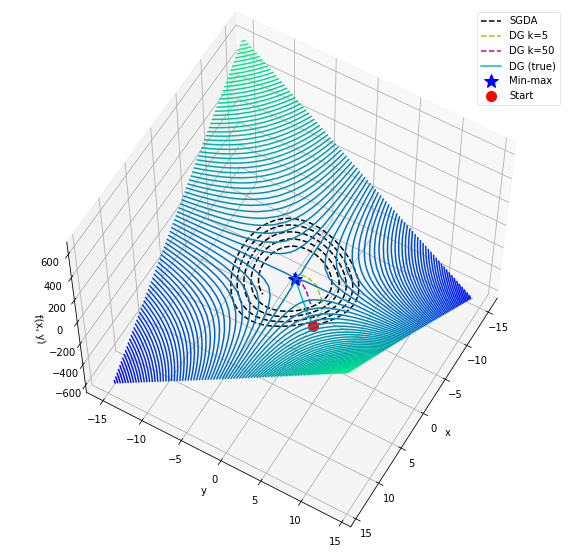}
  \caption{c=3}
\end{subfigure}%
\begin{subfigure}{.2\textwidth}
  \centering
  \includegraphics[width=1\textwidth]{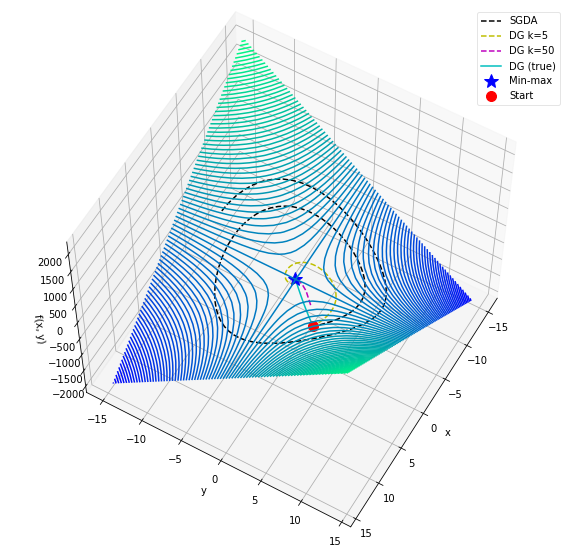}
  \caption{c=10}
\end{subfigure}
\caption{DG (and its approximations) converge, whereas GDA either oscillates or converges very slowly.}
\label{fig:ccandncc_toy}
\vspace{-5mm}
\end{figure}

% \begin{figure}
% \centering
% \begin{subfigure}{.3\textwidth}
%   \centering
%   \includegraphics[width=1\textwidth]{ICML_2021/figures/toy_sgda_vs_dg/cc_c3_toy.png}
%   \caption{c=3}
%   \label{fig:sub1}
% \end{subfigure}%
% \begin{subfigure}{.3\textwidth}
%   \centering
%   \includegraphics[width=1\textwidth]{ICML_2021/figures/toy_sgda_vs_dg/cc_c10_toy.png}
%   \caption{c=10}
%   \label{fig:sub2}
% \end{subfigure}
% \begin{subfigure}{.3\textwidth}
%   \centering
%   \includegraphics[width=1\textwidth]{ICML_2021/figures/toy_sgda_vs_dg/cc_c10_toy.png}
%   \caption{c=10}
%   \label{fig:sub2}
% \end{subfigure}
% \caption{a) and b) Convex-concave game: DG (and its approximations) converge, whereas GDA either oscillates or converges very slowly. c) Nonconvex-nonconcave game: DG still converges.}
% \label{fig:ccandncc_toy}
% \vspace{-5mm}
% \end{figure}

\vspace{-0.2cm}
\subsection{Stability behavior}
\vspace{-0.2cm}
Finally, we analyze the stability behavior of optimizing the duality gap for the game $f(x, y) = cxy, c \in \{3, 10\}$ shown in Fig.~\ref{fig:ccandncc_toy}. We provide a brief summary and give a detailed derivation in Appendix \ref{app:bilinear_games}.
\vspace{-0.3cm}
\paragraph{SGD}
The SGD updates are:
\begin{align*}
   \begin{bmatrix}
    x_{t+1} \\
    y_{t+1}
  \end{bmatrix}
=
   \begin{bmatrix}
    1 \quad -\eta c \\
    \eta c \quad \quad 1
  \end{bmatrix}
   \begin{bmatrix}
    x_t \\
     y_t
  \end{bmatrix}
\end{align*}
The eigenvalues of the system are $\lambda_1 = 1-i \cdot \eta c$ and $\lambda_2=1+i \cdot \eta c$, which leads either to oscillations or divergent behavior depending on the value of c. 
\vspace{-0.3cm}
\paragraph{Duality gap}
The duality gap function can be defined as $dg(x, y) = c(x \cdot y_w - x_w \cdot y)$ with updates:
\begin{align*}
   \begin{bmatrix}
    x_{t+1} \\
    y_{t+1}
  \end{bmatrix}
=
   \begin{bmatrix}
    1 -\eta^2 k c & - \eta \\
    \eta & 1 - \eta^2 k c
  \end{bmatrix}
  \begin{bmatrix}
    x_{t} \\
    y_{t}
  \end{bmatrix}
\end{align*}

The eigenvalues of the above system are
\begin{equation*}
\lambda_{1,2} = - \eta^2 k c + 1 \pm i \eta.
\end{equation*}

%The algorithm converges for all $\eta^2 <= \frac{2kc}{k^2c^2}$.
We therefore have stability as long as $\eta < \frac{2kc - 1}{k^2 c^2}$, which requires the number of updates to be $k > \frac{1}{2c}$. As can be seen by examining the DG updates with a finite $k$, there is an additional term that appears in the updates that "contracts" the dynamics and leads to convergence, even for $k=1$.

In Appendix~\ref{app:updates_toy}, we also analyze the stability of the DG updates for other games that appear in Fig.~\ref{fig:toyexample}. We observe the same behavior where DG is stable, in contrast to SGD and related alternatives.

%% file: 04_related_work.tex
\vspace{-3.5mm}
\section{Related work}
\vspace{-2mm}

Stabilizing GAN training has been an active research area: in proposing new objectives~\citep{arjovsky2017wasserstein}, adding regularizers~\citep{gulrajani2017improved, roth2017stabilizing} or designing better architectures~\citep{radford2015unsupervised}. In terms of optimization, the rotational dynamics of games has been pointed out as a cause of the complexity~\citep{mescheder2018training, balduzzi2018mechanics}, which was recently empirically demonstrated as well~\citep{berard2019closer}. To overcome oscillations various works have explored iterate or model averaging~\citep{gidel2018variational, grnarova2017online} or used second-order information~\citep{balduzzi2018mechanics, wang2019solving}. Other alternatives include training using optimism~\citep{daskalakis2017training} which extrapolates the next value of the gradient or~\citet{gidel2018variational} proposing a variant of extragradient that anticipates the opponent's actions, although was recently shown to break in simple settings~\citep{chavdarova2019reducing}.~\citet{razaviyayn2020nonconvex} give a survey of recent advances in min-max optimization.

The idea of duality for improving GAN training has been explored in simple settings; e.g~\cite{li2017dualing} show stabilization using the dual of linear discriminators. ~\cite{farnia2018convex} rely on duality to give a different interpretation to GANs with constrained discriminators and~\cite{gemici2018primal} use the dual formulation of WGANs to train the decoder. A recent approach~\citep{chen2018training} uses a type of Lagrangian duality in order to derive an objective to train GANs although it is not directly relatable to the original GAN as it is based on an adhoc assumption on finite set. Finally,~\citet{grnarova2019domain} proposed to use the duality gap, although not for optimization, but as an evaluation metric to monitor the training progress.

%% file: 05_experiments.tex
\vspace{-2.4mm}
\section{Experiments}
\label{sec:experiments}
\vspace{-2mm}

We turn to an empirical evaluation of our theory for a wide range of problems commonly discussed in the literature~\citep{berard2019closer, metz2016unrolled, wang2019solving}.

\subsection{Convergence analysis}
\vspace{-1mm}

As previously discussed, gradient descent ascent (GDA) and many related gradient-based algorithms exhibit undesirable stability properties when used for solving games. Most of these instabilities can in fact be observed on simple toy problems where we will start our investigation before moving on to more complex problems. First, we demonstrate two fundamental properties of our algorithm (i) convergence to (local) minima (which correspond to (local) Nash equilibria in the original game objective) while GDA either diverges or goes into limit cycles and (ii) avoiding convergence to bad critical points to which GDA is attracted to.

To that end, we compare DG to a variety of existing optimization algorithms: (i) Gradient Descent Ascent (GDA), (ii) Optimistic Gradient Descent Ascent (OGDA) ~\citep{daskalakis2017training}, (iii) Extragradient (EG)~\citep{korpelevich1976extragradient}, (iv) Symplectic Gradient Adjustment (SGA)~\citep{balduzzi2018mechanics}, (v) Concensus Optimization (CO)~\citep{mescheder2018training}, (vi) Unrolled SGDA~\citep{metz2016unrolled} and (vii) Follow-the-Ridge (FR)~\citep{wang2019solving} on three simple low-dimensional problems (Fig. \ref{fig:toyexample}). These functions were proposed in \cite{wang2019solving}:
\begin{align*}
    f_1(x, y) &= -3x^2 -y^2 +4xy, f_2(x, y) = 3x^2 + y^2 + 4xy & \\
    f_3(x, y) &= (4x^2 - (y - 3x + 0.05x^3)^2 -0.1y^4)e^{-0.01(x^2 + y^2)}.
\end{align*}
\vspace{-0.1cm}
The first two functions (see Fig. \ref{fig:toyexample}, left and middle panels) are two-dimensional quadratic problems while the third function (Fig. \ref{fig:toyexample} right) has a more complicated landscape due to the sixth-order polynomial being scaled by an exponential. The first function has a local (and global) minimax at (0, 0). In Fig. \ref{fig:toyexample} (left) it can be seen that only DG, FR, SGA and CO converge to it, while other methods diverge. For the second function, (0, 0) is not a local minimax (it is a local min for the max player); yet all algorithms except for DG, Unrolled GD and FR converge to this undesired stationary point. Finally, for the polynomial function (right), (0, 0) is again a local minimax, but most methods cycle around the equilibrium. Again, DG and FR are able to avoid the oscillating behaviour and converge to the correct solution. 

\begin{figure}[htp]
\centering
\begin{minipage}[c]{0.51\textwidth}
\includegraphics[width=.33\textwidth]{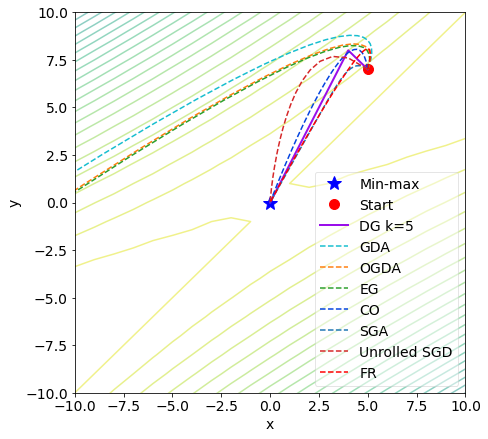}\hfill
\includegraphics[width=.33\textwidth]{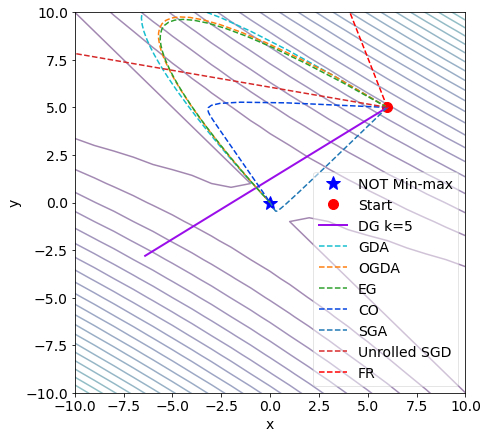}\hfill
\includegraphics[width=.33\textwidth]{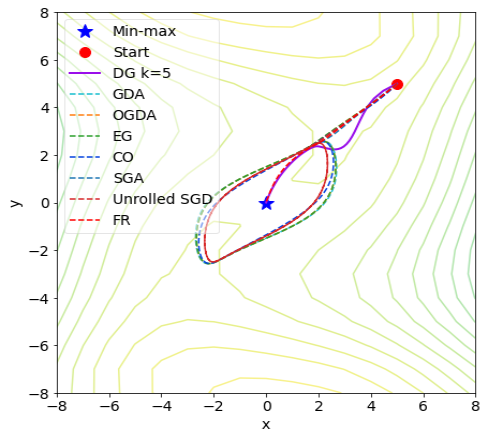}
\end{minipage}
\caption{\label{fig:toyexample} Contours represent function values. DG converges to the solution for left and right, while avoiding bad stationary point in the middle.}
\end{figure}

% \begin{figure*}[htp]
% \centering
% \begin{minipage}[c]{0.51\textwidth}
% \includegraphics[width=.33\textwidth]{ICML_2021/figures/experiments/toy_examples/sgo_1_color.png}\hfill
% \includegraphics[width=.33\textwidth]{ICML_2021/figures/experiments/toy_examples/sgo_2_color.png}\hfill
% \includegraphics[width=.33\textwidth]{ICML_2021/figures/experiments/toy_examples/sgo_3_color_resized.png}

% \caption{\label{fig:toyexample} Contours represent function values. DG converges to the solution for left and right, while avoiding bad stationary point in the middle.}
% \end{minipage}
% \begin{minipage}[c]{0.45\textwidth}
%     \includegraphics[width=0.6\linewidth]{ICML_2021/figures/experiments/optimization_landscape/GANs_vs_DGANs.png}
% \caption{\label{fig:gans_vs_dggans}MNIST Inception score}
% \end{minipage}
% \end{figure*}

Overall it can be observed that most existing algorithms fail on even simple toy examples, an observation made in previous works as well (e.g. \cite{wang2019solving, mescheder2018training, gidel2018variational}).  In contrast, we observe a positive behavior for the DG objective, both around good as well as bad stationary points. The reason for this is the "correction" term that appears due to the updates. For analytical analysis and further intuition, see Appendix~\ref{app:updates_toy}.

\vspace{-0.2cm}
\subsection{Generative Adversarial Networks}
\vspace{-0.2cm}

We now turn to the problem of training GANs using DG as an objective function. As discussed in Sec.~\ref{sec:theory}, we sample pairs $(x, z) \in \Xcal \times \Zcal$ such that they satisfy the condition $G(z)=x$, for some $G: \Zcal \to \Xcal$. This can be achieved by using an existing GAN architecture with an encoder such as BiGAN~\citep{donahue2016adversarial}, ALI~\citep{dumoulin2016adversarially} and BigBiGAN~\citep{donahue2019large}.
In a nutshell, the encoder $E: \Xcal \to \Zcal$ provides an inverse mapping by projecting data back into the latent space. The discriminator then not only classifies a real/generated datapoint, but instead receives as an input a pair of a datapoint and a corresponding latent vector - $(x, E(x))$ and $(G(z), z)$. In~\citep{donahue2019large}, the authors demonstrate that a certain reconstruction ensures the condition $G(z)=x$ holds with $z=E(x)$. We exploit such a construction for optimizing the DG objective. Halfway through the training, we start training with pairs $(G(E(x)), E(x))$ and $(x, E(x))$ in order to achieve the computation of the DG with pairs that satisfy the aforementioned condition. The number of update steps is chosen from $k \in [5, 10, 25]$ and we optimize using Adam. All training details and additional baselines can be found in Appendix~\ref{app:experiments}.

\vspace{-0.3cm}
\subsubsection{Mixture of Gaussians}

We first evaluate 5 different algorithms (GDA, ALI, EG, CO and DG) on a Gaussian mixture with the original GAN saturating loss in a setting shown to be difficult for standard GDA ~\citep{wang2019solving} (Fig.~\ref{fig:mog}). The model does not converge with GDA or EG and suffers from mode collapse. Only DG and CO are able to learn the data distribution, however DG quickly converges to the solution. We also include results for ALI that show its instability on this toy problem. In particular, the stability of all models can be seen by the progression of the DG metric throughout the training (last row in Fig.~\ref{fig:mog}). For DG, we see the gap quickly goes to zero. 

\begin{figure}
\centering
    \includegraphics[width=0.51\linewidth]{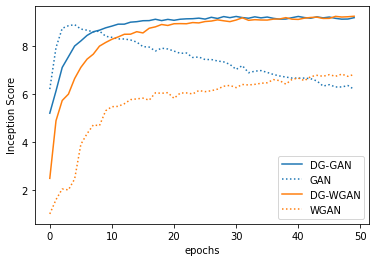}
\caption{\label{fig:gans_vs_dggans}MNIST Inception score}
\end{figure}

\vspace{-0.3cm}
\begin{figure*}
  \begin{minipage}[c]{0.63\textwidth}
    \includegraphics[width=0.98\textwidth]{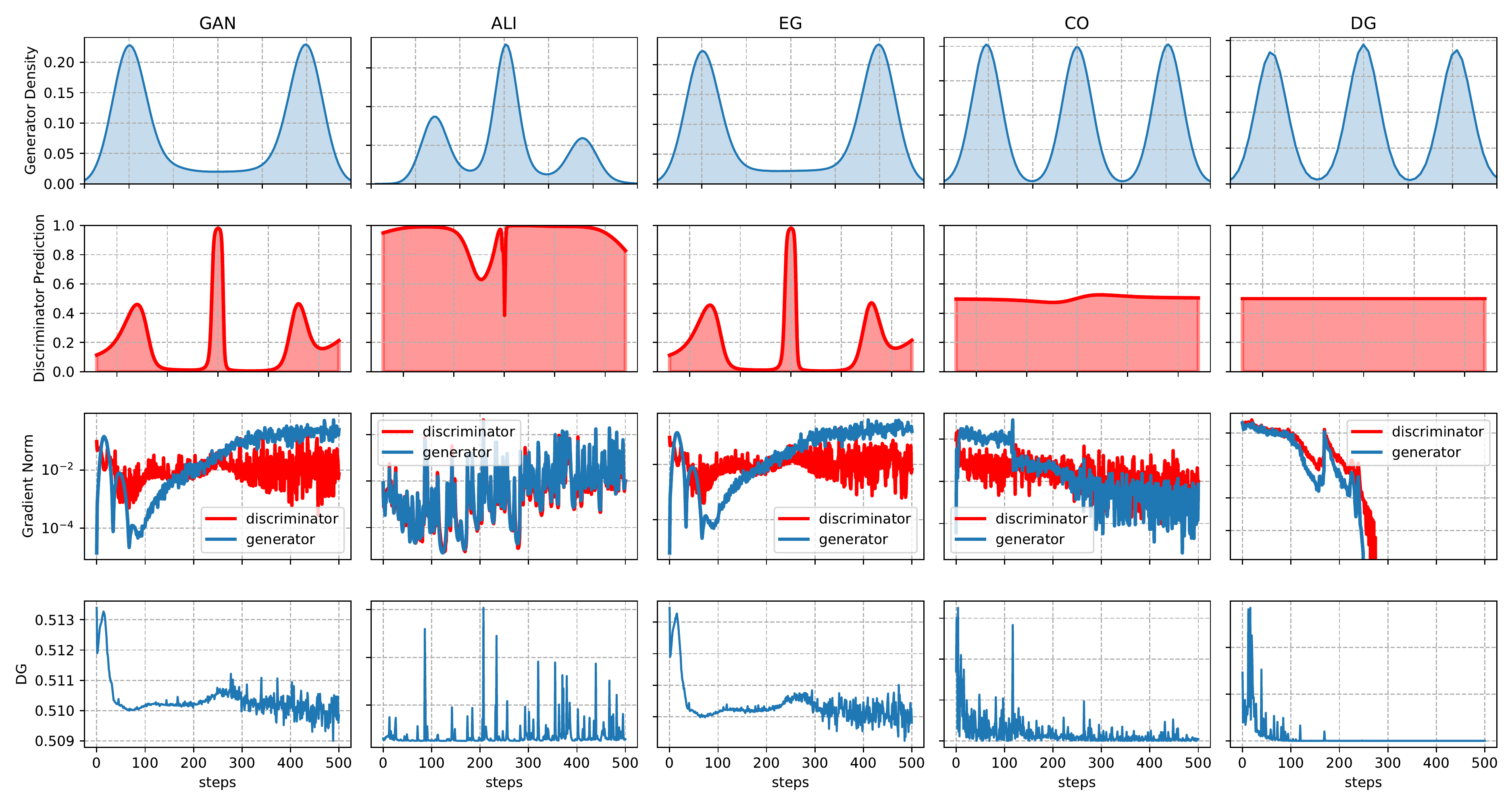}
  \end{minipage}\hfill
  \begin{minipage}[c]{0.37\textwidth}
    \caption{GANs with saturating loss on a Gaussian mixture. \textbf{Row 1}: Generator Density. Only CO and DG capture all modes. \textbf{Row 2}: Discriminator prediction. For a perfect generator, the discriminator outputs 0.5. \textbf{Row 3}: Gradient norm of the players. \textbf{Row 4}: Duality gap metric throughout training. DG quickly converges and displays stability.}
    \vspace{0.2cm}
    \label{fig:mog}
  \end{minipage}
\end{figure*}
\begin{figure}[htp]
\centering
\begin{subfigure}[b]{0.23\textwidth}
\includegraphics[width=\textwidth]{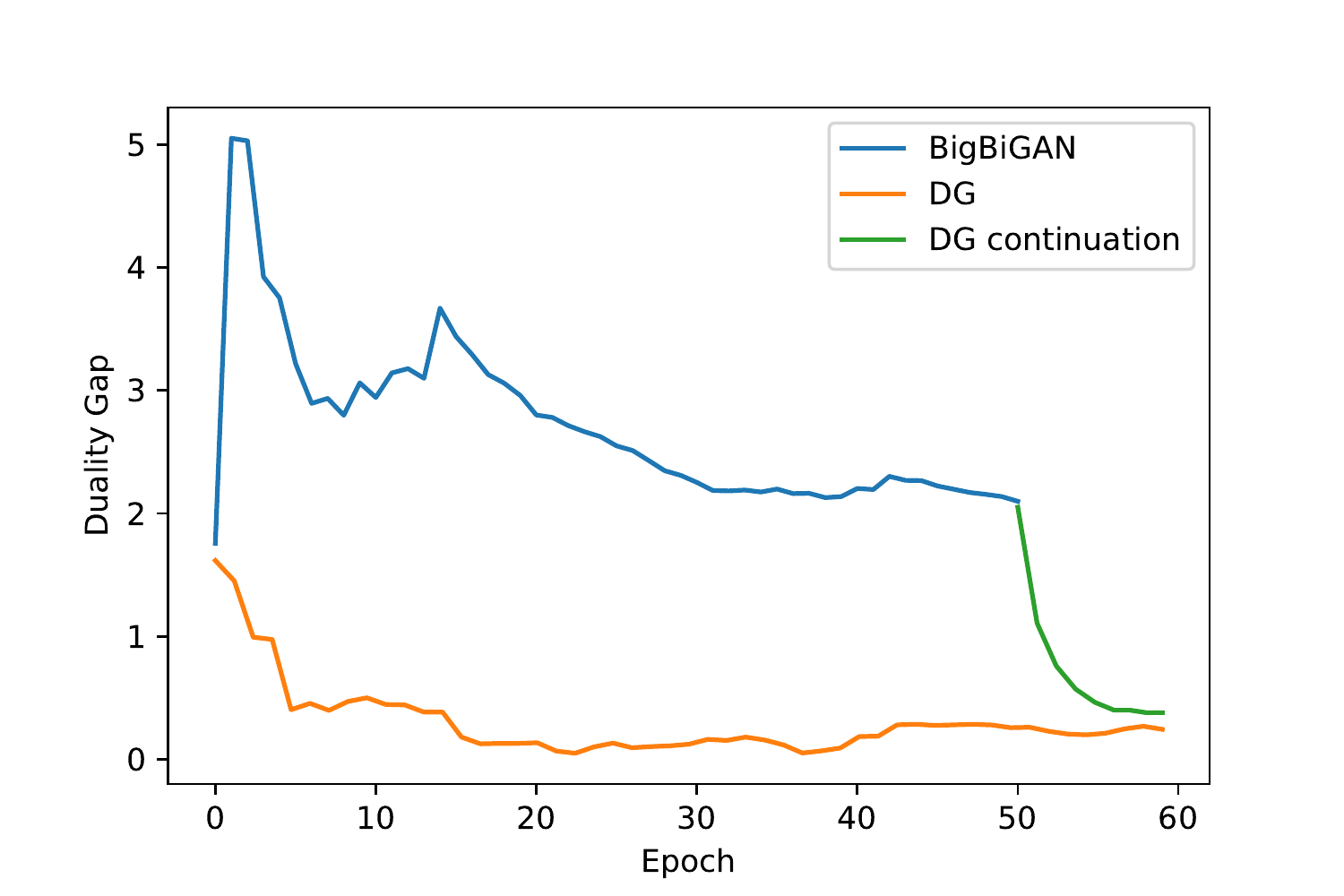}
\end{subfigure}
\begin{subfigure}[b]{0.23\textwidth}
\includegraphics[width=\textwidth]{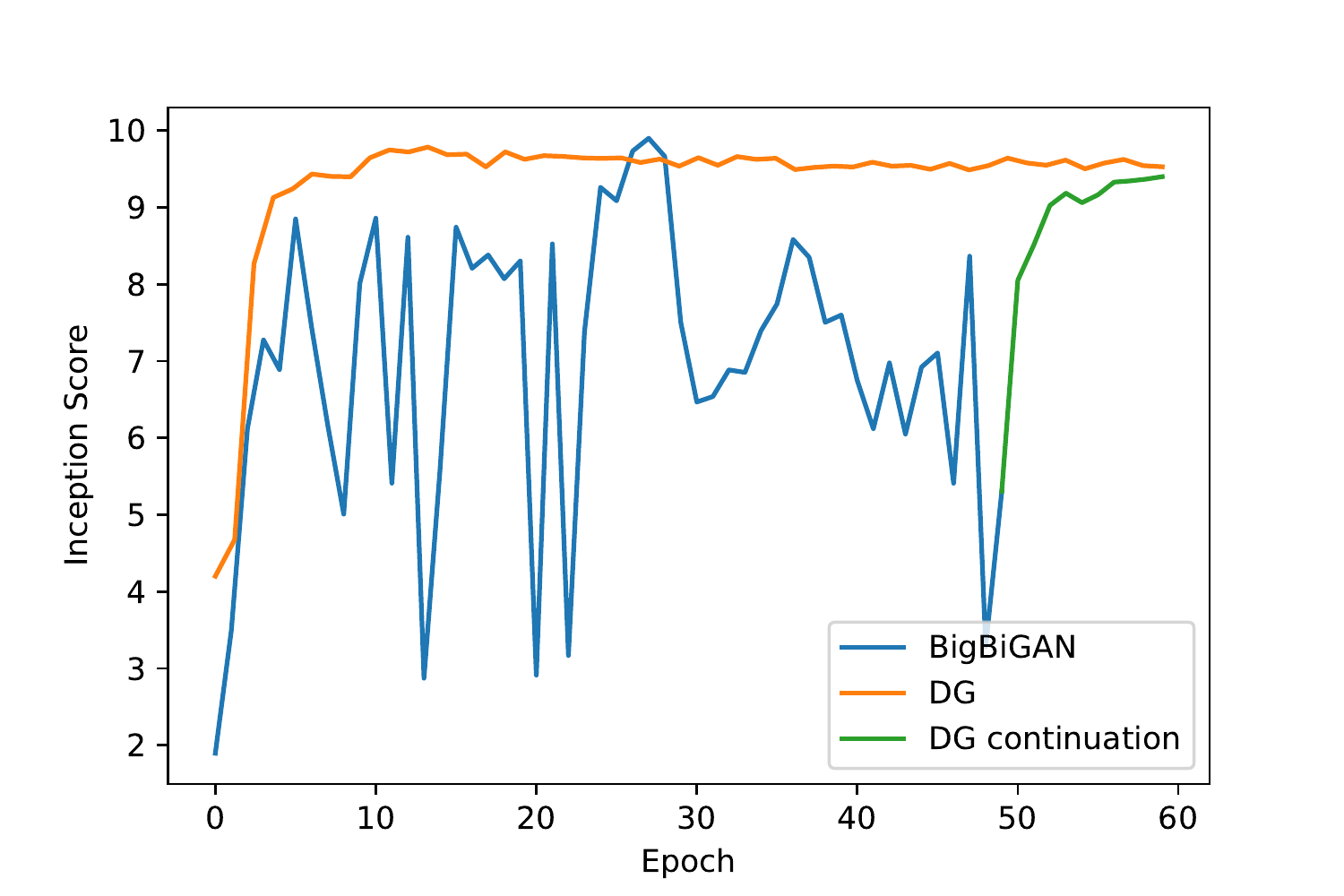}
\end{subfigure}
\begin{subfigure}[b]{0.23\textwidth}
\includegraphics[width=\textwidth]{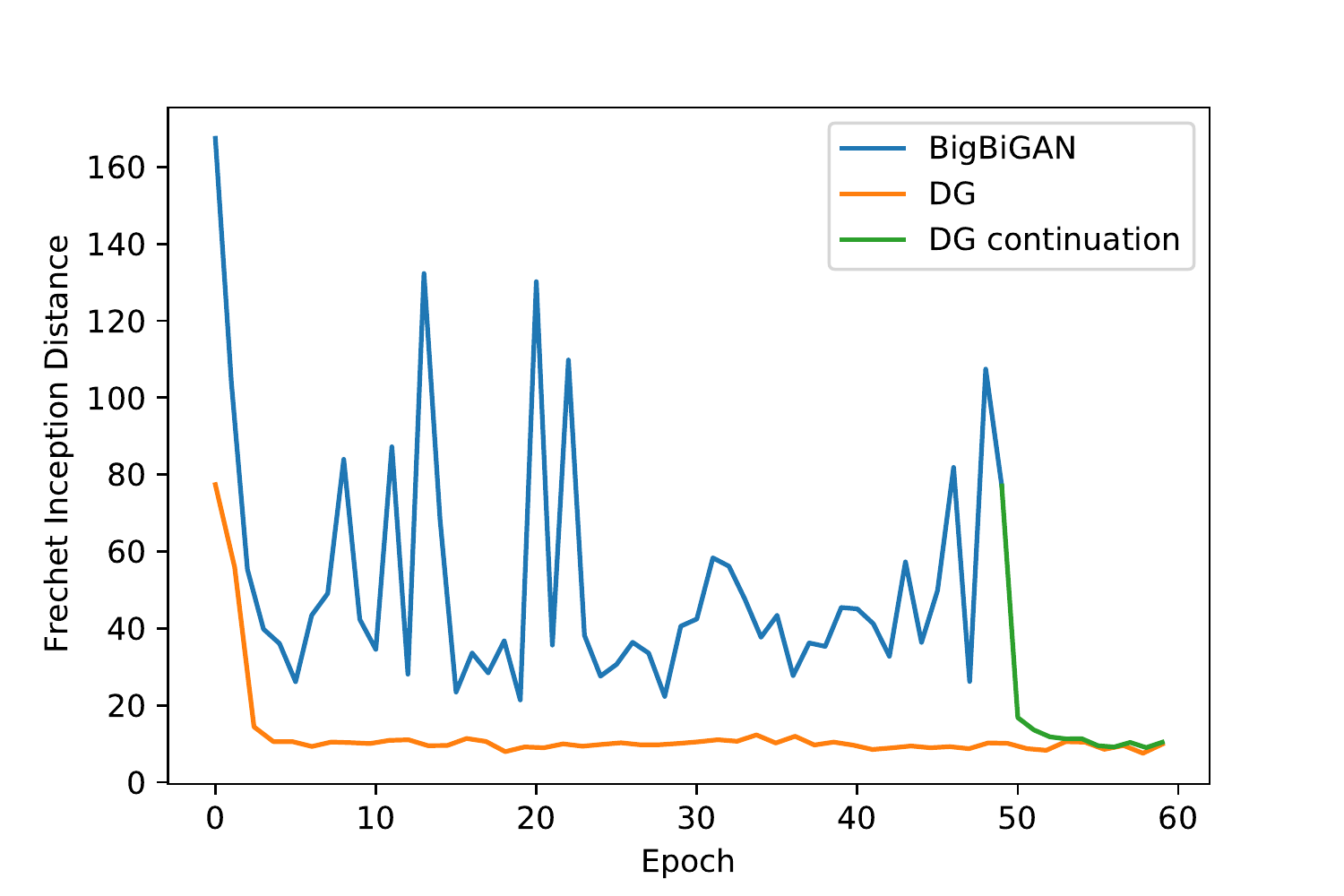}
\end{subfigure}
\begin{subfigure}[b]{0.22\textwidth}
\includegraphics[width=\textwidth]{ICML_2021/figures/experiments/bigbigan/lr_window.pdf}
\end{subfigure}
\caption{BigBiGAN (blue) vs. DG (orange) on MNIST. \textbf{a)} Duality Gap metric. \textbf{b)} MNIST IS. \textbf{c)} MNIST FID. The green curve corresponds to taking the last checkpoint of the BigBiGAN model once converged and further training it with the DG objective. \textbf{d)} Number of iterations until a model reaches IS of at least 8.0 across different learning rates.}
\label{fig:bigbigan}
\end{figure}
\vspace{-0.1cm}

\begin{table}[htp]
    \centering
    \resizebox{0.5\textwidth}{!}{
\begin{tabular}{@{}lc|cc@{}}
\toprule
\multicolumn{2}{c}{\textbf{GAN based}}                 & \multicolumn{2}{c}{\textbf{WGAN based}}       \\ \midrule
\multicolumn{1}{c}{\textbf{Model}} & \textbf{IS}       & \textbf{Model}              & \textbf{IS}      \\
GAN (Adam)                         & 8.58 $\pm$ 0.006  & \multicolumn{1}{l}{WGAN}    & 7.41 $\pm$ 0.029 \\
GAN (ExtraAdam)                    & 8.80 $\pm$ 0.021  &                             &                  \\
GAN (SGD)                          & 8.19 $\pm$ 0.017  &                             &                  \\
GAN (RMSProp)                      & 8.69 $\pm$ 0.013  &                             &                  \\
BiGAN/ALI                          & 8.65 $\pm$ 0.081  & \multicolumn{1}{l}{WGAN GP} & 9.28 $\pm$ 0.004 \\
DG                                 & \textbf{9.265 $\pm$ 0.021} & \multicolumn{1}{l}{DG WGAN} & \textbf{9.59 $\pm$ 0.014} \\ \bottomrule
\end{tabular}}
      \captionof{table}{\label{tab:is_score} MNIST Inception score averaged over 5 runs. The average IS score of real MNIST data is 9.9.}
\end{table}
\subsubsection{Generating images}
%\vspace{+0.5mm}

We showed that when training by optimizing DG, the algorithm (i) exhibits desirable convergence properties and (ii) is more stable. Next we look at how this leads to improvements in generating real images. 
%\vspace{-0.3cm}

\textbf{Exploring the optimization landscape.}
We train a GAN on MNIST with a DCGAN architecture~\citep{radford2015unsupervised} and spectral normalization. The adversarial losses we consider are: (i) variants of the GAN objective (optimized with different algorithms including SGD, Adam, RmsProp and ExtraAdam~\citep{gidel2018variational}) and (ii) several variants of WGAN~\citep{arjovsky2017wasserstein}  (and with gradient penalty WGAN-GP~\citep{gulrajani2017improved}). Note that our optimization setting can be applied to any minimax formulation of GANs by optimizing the DG for the specific game objective respectively (GAN, WGAN etc.). We denote DG applied to the GAN objective as DG GAN and to the WGAN objective as DG WGAN, and report the Inception Score (IS) ~\citep{salimans2016improved} in Tab.~\ref{tab:is_score} computed using an MNIST classifier. The optimization setup improves the scores for both objectives. In fact, there is a gap in the scores between the standard, adversarial, and our optimization setting throughout the entire training (see Fig.~\ref{fig:gans_vs_dggans}). 

In addition, we follow the setup from~\cite{berard2019closer} and investigate extensively the rotational and convergence behaviours of the models (See Fig.~\ref{fig:game_dynamics}). Overall, while GANs exhibit rotations and converge to points that are a saddle, instead of a local $\min$ for the generator, DG is more stable, converges to local NEs and avoids the recurrent dynamics, which is also reflected by the improved IS scores.
\vspace{-0.4cm}
\paragraph{Fréchet Inception Distance (FID).} In Tab.\ref{tab:different_datasets} we further compare DG to various GAN algorithms on different datasets (ranging from simple to advanced complexity): MM GAN (saturating GAN), NSGAN (non-saturating GAN), LSGAN~\citep{mao2017least}, WGAN, WGAN-GP, DRAGAN~\citep{kodali2017convergence}, BEGAN~\citep{berthelot2017began}, SNGAN~\citep{miyato2018spectral}, ExtraAdam~\citep{gidel2018negative} and DCGAN. FID ~\citep{heusel2017gans} is computed using 10K samples across 10 different runs using the features from the Inception Net for all datasets except CelebA, for which we use VGG. We again see that minimizing the DG translates to practical improvement through obtaining better (or comparable) results.

\hspace*{-0.8cm}
\begin{table}
\resizebox{0.49 \textwidth}{!}{
\begin{tabular}{lllll}
\toprule
Alg/FID   & MNIST      & FMNIST      & CIFAR10     & CelebA      \\ \midrule
MMGAN     & 9.8 ± 0.9  & 29.6 ± 1.6  & 72.7 ± 3.6  & 65.6 ± 4.2  \\
NSGAN     & 6.8 ± 0.5  & 26.5 ± 1.6  & 58.5 ± 1.9  & 55.0 ± 3.3  \\
LSGAN     & 7.8 ± 0.6  & 30.7 ± 2.2  & 87.1 ± 47.5 & 53.9 ± 2.8  \\
WGAN      & 6.7 ± 0.4  & 21.5 ± 1.6  & 55.2 ± 2.3  & 41.3 ± 2.0  \\
WGAN GP   & 20.3 ± 5.0 & 24.5 ± 2.1  & 55.8 ± 0.9  & 30.0 ± 1.0  \\
DRAGAN    & 7.6 ± 0.4  & 27.7 ± 1.2  & 69.8 ± 2.0  & 42.3 ± 3.0  \\
BEGAN     & 13.1 ± 1.0 & 13.1 ± 1.0  & 71.4 ± 1.6  & 38.9 ± 0.9  \\\midrule
SNGAN     & 6.5 ± 0.2  & 24.3 ± 1.2  & 19.22 ± 2.4 & 46.2  ± 3.4 \\
ExtraAdam & 6.6 ± 0.8  & 24.1 ± 0.8  & 17.31 ± 1.8 & 40.1 ± 2.8  \\
DCGAN     & 7.8 ± 0.3  & 29.01 ± 0.4 & 21.12 ± 2.0 & 54.6 ± 3.2  \\
DG (k=10) & 6.0 ± 0.6  & 20.03 ± 1.1 & 13.66 ± 2.2 & 30.1 ± 2.0 \\
\bottomrule
\end{tabular}
}
\caption{\label{tab:different_datasets} FID for various algorithms. Results from the first row taken from \cite{lucic2018gans}}.
\end{table}

\vspace{-0.3cm}
\textbf{Training BigBiGANs by optimizing DG.}
\label{sec:bigbigan}
As previously discussed, moving from an \textit{adversarial} to an \textit{optimization} setting presents several benefits which we highlight in the following experiment. We train a low-resolution BigBiGAN~\citep{donahue2019large} on MNIST, Fashion-MNIST and Cifar10 and report the corresponding FID for conditional and unconditional models trained via DG and the BigBiGAN objective. Tab.~\ref{tab:fid} shows that the FID scores for DG are substantially improved, consistently across all datasets and settings. In fact, one can demonstrate increased benefits of the optimization setting by further training the final converged models obtained with the adversarial loss. This yields significant improvements, as shown in the last rows of Table~\ref{tab:fid} (see GAN + DG).

% \begin{figure}[htp]
% \centering
% \begin{subfigure}[b]{0.23\textwidth}
% \includegraphics[width=\textwidth]{ICML_2021/figures/experiments/bigbigan/dgap_cifar10.pdf}
% \end{subfigure}
% \begin{subfigure}[b]{0.23\textwidth}
% \includegraphics[width=\textwidth]{ICML_2021/figures/experiments/bigbigan/mnist_Score_test.pdf}
% \end{subfigure}
% \begin{subfigure}[b]{0.23\textwidth}
% \includegraphics[width=\textwidth]{ICML_2021/figures/experiments/bigbigan/mnist_Frechet_Distance_test.pdf}
% \end{subfigure}
% \begin{subfigure}[b]{0.22\textwidth}
% \includegraphics[width=\textwidth]{ICML_2021/figures/experiments/bigbigan/lr_window.pdf}
% \end{subfigure}
% \caption{BigBiGAN (blue) vs. DG (orange) on MNIST. \textbf{a)} Duality Gap metric. \textbf{b)} MNIST IS. \textbf{c)} MNIST FID. The green curve corresponds to taking the last checkpoint of the BigBiGAN model once converged and further training it with the DG objective. \textbf{d)} Number of iterations until a model reaches IS of at least 8.0 across different learning rates.}
% \label{fig:bigbigan}
% \end{figure}
\vspace{-0.1cm}
\begin{table}[htp]
\small
\resizebox{0.49 \textwidth}{!}{
\begin{tabular}{cccc}
\toprule
       & \textbf{CIFAR U} & \textbf{MNIST U} & \textbf{FMNIST U} \\
       \midrule
DG     & 22.14   & 6.11    & 20.33    \\
GAN    & 24.41   & 10.56   & 22.55    \\
GAN+DG & 22.26   & 6.76    & 20.41    \\
       & \textbf{CIFAR C} & \textbf{MNIST C} & \textbf{FMNIST C} \\
       \midrule
DG     & 21.45   & 6.05    & 20.21    \\
GAN    & 23.00   & 22.03   & 25.86    \\
GAN+DG & 21.42   & 9.91    & 21.19  \\
\bottomrule

\end{tabular}
}
\caption{\label{tab:fid} FID for a BigBiGAN trained by optimizing DG and the adversarial loss. The FID for the third rows corresponds to models initialized with the pretrained GAN weights from rows (2), further trained via the DG. (U) -  unconditional (C) - conditional model.}
\end{table}

\vspace*{-0.1cm}
DG again exhibits more stable behaviour throughout the training, as can be seen in Fig.~\ref{fig:bigbigan} a by the progression of the DG when trained on MNIST. There is a consistent gap in terms of DG for the models that can be improved by further training the BigBiGAN with the DG objective. We observe similar behaviour across the different datasets (additional plots in App.~\ref{app:experiments}). Moreover, we show the progression of the MNIST IS and MNIST FID for the two models, which again point out to DG being more stable (Fig.~\ref{fig:bigbigan} b and c).

Apart from improved stability and performance, this experiment also demonstrates another property of optimizing without the competitive aspect of the minimax objective. Since the DG becomes lower when either player gets better (irrespective of the other player), there is no competition between the two players as in standard GAN training. Instead, in GANs one player is trying to maximize, while the other is trying to minimize the objective, so in order for one player to improve their utility, they need to do better with respect to their opponent. This ultimately means that one needs to carefully adjust the learning rates of both players; if one player is dominant, the game may become unstable and stop at a suboptimal point~\citep{heusel2017gans}. In contrast, optimizing the DG (i.e. both players have the same goal), relaxes the need for tuning learning rates (Fig.~\ref{fig:bigbigan} d).
\vspace{-0.2cm}

%% file: 06_conclusion.tex
\vspace{-2mm}
\section{Conclusion}
\vspace{-2mm}

Training GANs is a notoriously hard problem due to the min-max nature of the GAN objective. Practitioners often face a difficult hyper-parameter tuning step having to ensure none of the players overpowers the other. In this work, we proposed an alternative objective function based on the theoretically motivated concept of duality. We proved convergence of this algorithm under commonly used assumptions in the literature and we further supported our claim on a wide range of problems. Empirically, we have seen that optimizing the duality gap yields a more stable algorithm. An interesting direction for future work would be to loosen the convexity assumptions on one side of $M$~\citep{nouiehed2019solving}. Further relaxations have not yet been shown to be possible and are still an open question. Finally, one could explore alternative optimization methods for optimizing the DG such as second-order methods.

%% file: 07_appendix.tex
\appendix
\onecolumn
\section*{Appendix}
\section{Approximate Realizability}
Recall that in Section~\ref{sec:theory} we show that under the perfect realizability assumption (Assumption~\ref{assum:realizability}) then the original stochastic minimmax problem of Eq.~\eqref{eq:MinimaxStochastic} is equivalent to solving the stochastic DG problem of Eq.~\eqref{eq:DG_Stochastic}. And this facilitates the use of the stochastic DG formulation for solving the original problem.

Here we relax the assumption of \emph{perfect realizability} and show that under an \emph{approximate realizability assumption} one can still relate the solution of the stochastic DG problem to the original minimax objective.

First let us define approximate realizability with respect to the stochastic minimax problem (Eq.~\eqref{eq:MinimaxStochastic}),
\begin{Ass}[Approximate Realizability] \label{assum:ApproxRealizability}
There exists a solution $(\u^*,\v^*)$, and $\epsilon\geq 0$, such that $(\u^*,\v^*)$ is  an $\epsilon$-approximate equilibrium of $M(\cdot,\cdot;z)$ for any $z\in \text{supp}(\mP)$, i.e.,
$$
\DG(\u^*,\v^*;z)\leq \epsilon~,
$$
where  \text{supp}($\mP$) is the support of $\mP$.
\end{Ass}
Note that taking $\epsilon=0$ in the above definition is equivalent to the perfect realizability assumption (Assumption~\ref{assum:realizability}).

Clearly, under Assumption~\ref{assum:ApproxRealizability}, the original stochastic minimax  problem is no longer equivalent to the stochastic DG minimization problem.
Nevertheless, next we  show that in this case, solving the stochastic DG minimization problem will yield a solution which is $\epsilon$-optimal with respect to the original minimax problem.  
\begin{Lem}
Consider the stochastic minimax problem of Eq.~\eqref{eq:MinimaxStochastic}, and assume that the approximate realizability assumptions holds for this problem with some  $\epsilon\geq0$.
Also, Let $(\u_{\min},\v_{\min})$ be the minimizer of stochastic DG problem, i.e.,~
$(\u_{\min},\v_{\min}) \in~\arg\min_{\u,\v}~\E_z \DG(\u,\v;z)$. 
Then $(\u_{\min},\v_{\min})$ is an $\epsilon$-approximate equilibrium of the original minimax problem (Eq.~\eqref{eq:MinimaxStochastic}).
\end{Lem}
\begin{proof} Let $\DG(\cdot,\cdot)$ be the duality gap of the original minimax problem. Our goal is to show that $\DG(\u_{\min},\v_{\min}) \leq \epsilon$. Indeed, 
\begin{align*}
\DG(\u_{\min},\v_{\min})&: = \max_{\u,\v}  \E_z [M(\u_{\min},\v;z) -  M(\u,\v_{\min};z) ] \\
&\leq
\E_z \max_{\u,\v}[M(\u_{\min},\v;z) -  M(\u,\v_{\min};z) ] 
:= \E_z \DG(\u_{\min},\v_{\min};z) \leq \E_z \DG(\u^*,\v^*;z)\leq \epsilon
\end{align*}
where we have used the definition of $(\u_{\min},\v_{\min})$ as the minimizer of $\E_z \DG(\cdot,\cdot;z)$, as well as the $\epsilon$-realizability assumption.
\end{proof}
\section{Proof of Lemma~\ref{lem:AdaGrad}}
Before we prove the lemma, we denote $X = U \times V$, and for any 
$(\u,\v) \in U \times V$ we denote,
$$
x: = (\u,\v)~.
$$
Next we rewrite the problem formulation and our assumptions using the above notation.

Our objective can now be written as follows, 
$$
\min_{x\in X} \left[ \DG(x): = \E_{z\sim\mP}\DG(x;z) \right]
$$
We assume that each $\DG(\cdot;z)$'s is convex and $L$-smooth meaning,
$$
\|\nabla \DG(x;z) - \nabla \DG(y;z)\|\leq L\|x-y\|~, \quad \forall x,y\in X, z \in \text{supp}(\mP)
$$
where for any $x= (\u,\v)\in U\times V$ we denote,
$$
\nabla \DG(x;z): = \nabla_{(\u,\v)}\DG(\u,\v;z) = (\nabla_{\u}\DG(\u,\v;z), \nabla_{\v}\DG(\u,\v;z))
$$

We further assume there exists $x^*:= (\u^*,\v^*) \in U\times V$ such,
$$
\nabla \DG(x^*;z) = 0 ~, \quad \forall z \in \text{supp}(\mP)
$$

We analyze the following version of AdaGrad,
$$
x_{t+1} = \Pi_X (x_t - \eta_t g_t)~,\text{where}~~~ \eta_t =\frac{D}{\sqrt{\sum_{\tau =1}^t \|g_\tau\|^2}}~,
$$
where $g_t : = \nabla \DG(x_t;z_t)$ is an unbiased gradient estimate of $\nabla \DG(x_t)$, which uses a sample $z_t\sim\mP$. 
We also assume that $z_1,z_2,\ldots$ are samples i.i.d. from $\mP$.
Note that $\Pi_X(\cdot)$ is the orthogonal projection onto $X : = U\times V$, defined as $\Pi_X(y) =\arg \min_{x\in X} \|y-x\|~, \forall y\in\reals^d$.
Finally, $D$ denotes the diameter of $X$, i.e., $D = \max_{x,y\in X}\| y-x\|$.

Next we restate the lemma using these notations and provide the proof.

\begin{Lem}[Lemma~\ref{lem:AdaGrad}, Restated]\label{lem:AdaGradRephrase}
Consider a stochastic optimization problem of the form $\min_{x\in X}\DG(x) : = \E_{z\sim \mP} \DG(x;z)$. Further assume that $\DG(\cdot;z)$ is $L$-smooth and convex for any $z$ in the support of $\mP$, and that realizabilty assumption holds, i.e., there exists $x^*\in X$ such that $ \nabla \DG(x^*;z) = 0$ for any $z$ in the support of $\mP$. Then applying AdaGrad to this objective ensures an $O(L D^2/T)$ convergence rate where $D: = \max_{x,y\in X}\|x-y\|$ is the diameter of $X$. 
\end{Lem}

\begin{proof}[Proof of Lemma~\ref{lem:AdaGradRephrase}]
\textbf{Step 1.}
Our first step is to bound the second moment of the  gradient estimates.
To do so we shall require the following lemma regarding smooth functions (see proof in Appendix~\ref{sec:smallProof}),
\begin{Lem}\label{lemma:SmoothnessGrad}
Let $Q(x)$ be  $L$-smooth function and let $x^*$ be its global minima, i.e. 
$\nabla Q(x^*) =0$.Then the following holds,
$$
\|\nabla Q(x)\|^2   \leq 2L (Q(x) - Q(x^*))
$$
\end{Lem}
Applying the above lemma, and using realizability immediately implies,
\begin{align*}
\E [\|g_t\|^2\vert x_t] 
&= \E_{z_t \sim \mP} [\|\nabla \DG(x_t;z_t)\|^2 \vert x_t]  \\
&\leq
\E_{z_t \sim \mP} 2L [\DG(x_t;z_t) - \DG(x^*;z_t)\vert x_t] \\
&=
2L(\DG(x_t) -\DG(x^*))~.
\end{align*}
And therefore,
\begin{align}\label{eq:Variance}
\E [\|g_t\|^2] 
\leq 
2L\E(\DG(x_t) -\DG(x^*))~.
\end{align}

\textbf{Step 2.}  
Standard analysis of AdaGrad gives the following  bound (see e.g.; \cite{levy2017online}),
\begin{align*}
\sum_{t=1}^T \E[\DG(x_t)-\DG(x^*)] \leq \E \sqrt{\sum_{t=1}^T \|g_t\|^2}
\end{align*}
Using the above together with  Jensen's inequality with respect to the concave function $H(u): = \sqrt{u}~; u\in\reals_{+}$, and together with Eq.~\eqref{eq:Variance} gives,
\begin{align*}
\sum_{t=1}^T\E(\DG(x_t) - \DG(x^*))
& \leq 
D\E\sqrt{2\sum_{t=1}^T \|g_t\|^2} \\
&\leq
D\sqrt{2\sum_{t=1}^T \E\|g_t\|^2} \\
&\leq
D\sqrt{4L\sum_{t=1}^T \E(\DG(x_t) - \DG(x^*))} \\
\end{align*}
Re-arranging the above immediately gives,
$$
\sum_{t=1}^T\E(\DG(x_t) - \DG(x^*)) \leq 4LD^2~.
$$
Taking $\bar{x}_T: = \frac{1}{T}\sum_{t=1}^T x_t$ and using the above together with Jensen's inequality implies,
$$
\E(\DG(\bar{x}_T) - \DG(x^*)) \leq \frac{1}{T}\sum_{t=1}^T\E(\DG(x_t) - \DG(x^*)) \leq \frac{4LD^2}{T}~.
$$
which establishes an $O(1/T)$ rate for this case.

\end{proof}

\subsection{Proof of Lemma~\ref{lemma:SmoothnessGrad}} \label{sec:smallProof}

\begin{proof}
The $L$ smoothness of $Q$ means the following to hold $\forall x,u\in\reals^d$,
$$Q(x+u) \leq Q(x) +\nabla Q(x)^\top u+\frac{L}{2}\|u\|^2 ~.$$
Taking  $u=-\frac{1}{L}\nabla Q(x)$ we get,
$$Q(x+u) \le Q(x) -\frac{1}{L}\|\nabla Q(x)\|^2+\frac{1}{2L}\|\nabla Q(x)\|^2~.$$
Thus:
\begin{align*}
\|\nabla Q(x)\|^2 &\le 2L \big( Q(x) -Q(x+u)\big)\\
&  \le  2L \big(Q(x) -Q(x^*)\big)~,
\end{align*}
where in the last inequality we used $Q(x^*) \leq Q(x+u)$ which holds since $x^*$ is the \emph{global} minimum.
\end{proof}

\section{Bilinear Games}
\label{app:bilinear_games}
Throughout the paper we use several bilinear toy games. Concretely, Fig.~\ref{fig:motivation} is created using the function:
\begin{align*}
    f(x, y) = x^2  - y^2  + xy + 10 \sin(5x)+12*\sin(3y);
\end{align*}
for $x, y \in [-10, 10]$. \\
The functions used for Fig.~\ref{fig:ccandncc_toy} and Fig.~\ref{app:ccandncc_toy} are:
\begin{itemize}
    \item[] Convex-concave:
    \begin{align}
    \label{eq:cc_app}
        f(x, y) = cxy, c \in \{3, 10\}
    \end{align}
        \item[] Nonconvex-nonconcave:
    \begin{align}
        f(x, y) = F(x) + cxy - F(x), where \\
        F(x)= 
\begin{cases}
    -3(x + \frac{\pi}{2}),& \text{if } x< -\frac{\pi}{2}\\
    -3\cos(x),              & \text{if } -\frac{\pi}{2} \leq x\leq \frac{\pi}{2}\\
    -\cos(x) + 2x - \pi,& \text{if } x> \frac{\pi}{2}
\end{cases}
    \end{align}
\end{itemize}
as suggested in~\cite{abernethy2019last}.

Fig. \ref{fig:ncnc_toy_dist} also shows the behavior of the algorithm for different values of $k$. 
\begin{wrapfigure}{r}{5cm}
  \vspace{-0.3cm}
  \begin{center}
  
    \includegraphics[width=0.32\textwidth]{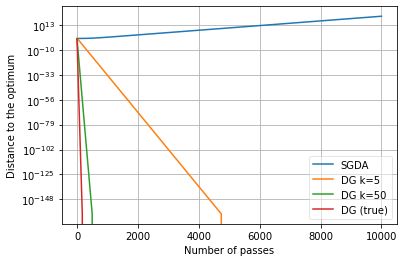}
  \end{center}
  %\caption{\label{fig:ncnc_toy_dist} \small{Distance to  solution for different $k$.} }
    \caption{\label{fig:ncnc_toy_dist} \small{Distance to the solution for different values of $k$.} }
  \vspace{-0.5cm}
\end{wrapfigure}

% \begin{figure}
% \centering
% \begin{subfigure}{.3\textwidth}
%   \centering
%   \includegraphics[width=1\textwidth]{ICML_2021/figures/toy_sgda_vs_dg/cc_c3_toy.png}
%   \caption{c=3}
%   \label{fig:sub1}
% \end{subfigure}%
% \begin{subfigure}{.3\textwidth}
%   \centering
%   \includegraphics[width=1\textwidth]{ICML_2021/figures/toy_sgda_vs_dg/cc_c10_toy.png}
%   \caption{c=10}
%   \label{fig:sub2}
% \end{subfigure}
% \begin{subfigure}{.3\textwidth}
%   \centering
%   \includegraphics[width=1\textwidth]{ICML_2021/figures/toy_sgda_vs_dg/cc_c10_toy.png}
%   \caption{c=10}
%   \label{fig:sub2}
% \end{subfigure}
% \caption{a) and b) Convex-concave game: DG (and its approximations) converge, whereas GDA either oscillates or converges very slowly. c) Nonconvex-nonconcave game: DG still converges.}
% \label{fig:ccandncc_toy}
% \vspace{-5mm}
% \end{figure}

In Fig.~\ref{fig:ccandncc_toy} we empirically demonstrate that using DG as an objective in a convex-concave setting (Fig.~\ref{fig:ccandncc_toy} (a) and b)), and a simple nonconvex-nonconcave setting (Fig.~\ref{app:ccandncc_toy} (c)) yields convergence to the solution of the game.

\begin{figure}
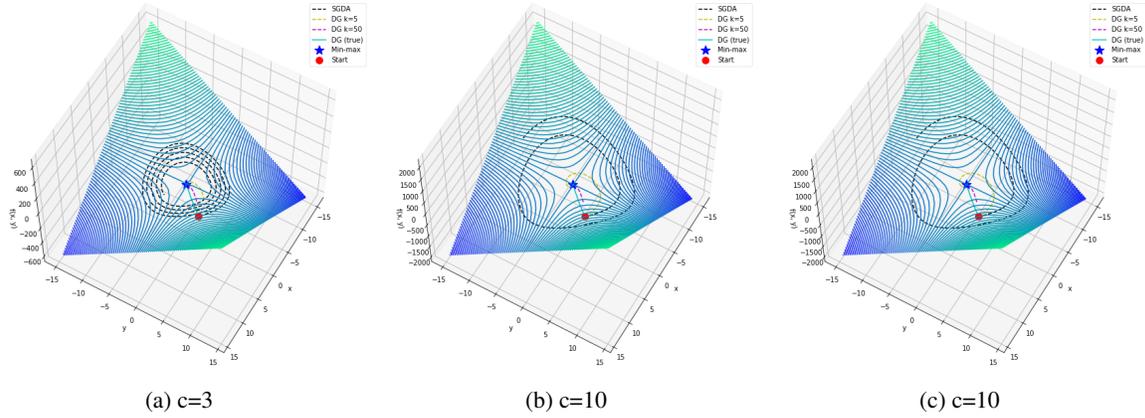

\centering
\begin{subfigure}{.3\textwidth}
  \centering
  \includegraphics[width=1\textwidth]{ICML_2021/figures/toy_sgda_vs_dg/cc_c3_toy.png}
  \caption{c=3}
  \label{fig:sub1}
\end{subfigure}%
\begin{subfigure}{.3\textwidth}
  \centering
  \includegraphics[width=1\textwidth]{ICML_2021/figures/toy_sgda_vs_dg/cc_c10_toy.png}
  \caption{c=10}
  \label{fig:sub2}
\end{subfigure}
\begin{subfigure}{.3\textwidth}
  \centering
  \includegraphics[width=1\textwidth]{ICML_2021/figures/toy_sgda_vs_dg/cc_c10_toy.png}
  \caption{c=10}
  \label{fig:sub2}
\end{subfigure}
\caption{a) and b) Convex-concave game: DG (and its approximations) converge, whereas GDA either oscillates or converges very slowly. c) Nonconvex-nonconcave game: DG still converges.}
\label{app:ccandncc_toy}
\vspace{-5mm}
\end{figure}

\paragraph{Analysing the updates.}
The SGD updates for the game Equation~\ref{eq:cc_app} are:
\begin{align*}
        [
   \begin{array}{l}
    x_{t+1} \\
    y_{t+1}
  \end{array}
] = [
   \begin{array}{l}
    x_{t} \\
    y_{t}
  \end{array}
] - \eta [
   \begin{array}{l}
    \nabla_x f(x_t, y_t) \\
    -\nabla_y f(x_t, y_t)
  \end{array}
] = \\
[
   \begin{array}{l}
    x_{t} \\
    y_{t}
  \end{array}
] - [
   \begin{array}{l}
    \eta y_{t} c \\
    -\eta x_{t} c
  \end{array}
].
    \end{align*}
Concretely, 
\begin{align*}
        [
   \begin{array}{l}
    x_{t+1} \\
    y_{t+1}
  \end{array}
] = [
   \begin{array}{l}
    1 \quad -\eta c \\
    \eta c \quad \quad 1
  \end{array}
] [
   \begin{array}{l}
    x_t \\
     y_t
  \end{array}
].
\end{align*}
The eigenvalues of the system are $\lambda_1 = 1-i \cdot \eta c$ and $\lambda_1=1+i*\eta*c$, which leads either to oscillations or divergent behavior depending on the value of c ~\citep{mescheder2018training}. 

The duality gap function can be defined as $dg(x, y) = c(x \cdot y_w - x_w \cdot y)$ and the corresponding updates are:
\begin{align*}
        [
   \begin{array}{l}
    x_{t+1} \\
    y_{t+1}
  \end{array}
] = [
   \begin{array}{l}
    x_{t} \\
    y_{t}
  \end{array}
] - \eta [
   \begin{array}{l}
    \nabla_x dg(x_t, y_w) \\
    -\nabla_y dg(x_w, y_t)
  \end{array}
] = \\
[
   \begin{array}{l}
    x_{t} \\
    y_{t}
  \end{array}
] - [
   \begin{array}{l}
    \eta y_{w} \cdot c \\
    -\eta x_{w} \cdot c
  \end{array}
], \text{where }
    \end{align*}
\begin{align*}
        [
   \begin{array}{l}
    x_{w} \\
    y_{w}
  \end{array}
] = [
   \begin{array}{l}
    x_{t} \\
    y_{t}
  \end{array}
]  - [
   \begin{array}{l}
    \eta y_{t} c k \\
    -\eta x_{t} c k
  \end{array}
].
    \end{align*}
Finally, the two combined lead to updates of the form:
\begin{align*}
   \begin{bmatrix}
    x_{t+1} \\
    y_{t+1}
  \end{bmatrix}
=
   \begin{bmatrix}
    x_{t} - \eta y_t -\eta^2 k c x_t \\
    y_{t} + \eta x_t -\eta^2 k c y_t
  \end{bmatrix} \\
=
   \begin{bmatrix}
    1 -\eta^2 k c & - \eta \\
    \eta & 1 - \eta^2 k c
  \end{bmatrix}
  \begin{bmatrix}
    x_{t} \\
    y_{t}
  \end{bmatrix}
\end{align*}
The algorithm converges for all $\eta^2 <= \frac{2kc}{k^2c^2}$. As can be seen, when optimizing the practical DG there is an additional term that appears in the updates that "contracts" everything and leads to convergence, even for $k=1$.

To study stability, we need to find the eigenvalues of the above matrix which we denote by $\Am$. We set 
\begin{align}
|\Am - \lambda \Im| = \left|
  \begin{bmatrix}
    1 -\eta^2 k c - \lambda & - \eta \\
    \eta & 1 - \eta^2 k c - \lambda
  \end{bmatrix} \right| = 0
\end{align}
% eigenvalues {(1-x*a^2*b, a), (-a, 1-x*a^2*b)}
The eigenvalues are
\begin{equation}
\lambda_{1,2} = - \eta^2 k c + 1 \pm i \eta
\end{equation}

% https://math.libretexts.org/Bookshelves/Applied_Mathematics/Book%3A_Introduction_to_the_Modeling_and_Analysis_of_Complex_Systems_(Sayama)/05%3A_DiscreteTime_Models_II__Analysis/5.06%3A_Asymptotic_Behavior_of_Discrete-Time_Linear_Dynamical_Systems
In order to get stability, we need the modulus of $\lambda_i < 1$, i.e.
\begin{align}
\sqrt{(- \eta^2 k c + 1)^2 + \eta^2} < 1,
\end{align}
which holds for $\eta < \frac{2kc - 1}{k^2 c^2}$.

\section{Experiments}
\label{app:experiments}
In the following we give experimental details such as hyperparameters, architectures and additional baselines.

\subsection{Toy problems}
The three low dimensional problems we consider are:
\begin{align*}
    f_1(x, y) &= -3x^2 -y^2 +4xy  & \\
    f_2(x, y) &= 3x^2 + y^2 + 4xy & \\
    f_3(x, y) &= (4x^2 - (y - 3x + 0.05x^3)^2 -0.1y^4)e^{-0.01(x^2 + y^2)}.
\end{align*}
as suggested in~\citep{wang2019solving}.
The algorithms we compare with are:
\begin{itemize}
    \item[(i)] Gradient Descent Ascent (GDA)
    \begin{align*}
        [
   \begin{array}{l}
    x_{t+1} \\
    y_{t+1}
  \end{array}
] = [
   \begin{array}{l}
    x_{t} \\
    y_{t}
  \end{array}
] - \eta [
   \begin{array}{l}
    \nabla_x f(x_t, y_t) \\
    -\nabla_y f(x_t, y_t)
  \end{array}
]
    \end{align*}
    
        \item[(ii)] Optimistic Gradient Descent Ascent (OGDA) ~\citep{daskalakis2017training}
    \begin{align*}
        [
   \begin{array}{l}
    x_{t+1} \\
    y_{t+1}
  \end{array}
] = [
   \begin{array}{l}
    x_{t} \\
    y_{t}
  \end{array}
] - 2\eta [
   \begin{array}{l}
    \nabla_x f(x_t, y_t) \\
    -\nabla_y f(x_t, y_t)
  \end{array}
] + \eta [
   \begin{array}{l}
    \nabla_x f(x_{t-1}, y_{t-1}) \\
    -\nabla_y f(x_{t-1}, y_{t-1})
  \end{array}
]
    \end{align*}
            \item[(iii)] Extragradient (EG) ~\citep{korpelevich1976extragradient}
    \begin{align*}
        [
   \begin{array}{l}
    x_{t+1} \\
    y_{t+1}
  \end{array}
] = [
   \begin{array}{l}
    x_{t} \\
    y_{t}
  \end{array}
] - \eta [
   \begin{array}{l}
    \nabla_x f(x_t - \eta \nabla_x f(x_t, y_t), y_t + \eta \nabla_y f(x_t, y_t)) \\
    -\nabla_y f(x_t - \eta \nabla_x f(x_t, y_t), y_t + \eta \nabla_y f(x_t, y_t))
  \end{array}
] 
    \end{align*}
    
                \item[(iv)] Symplectic Gradient Adjustment (SGA) ~\citep{balduzzi2018mechanics}
    \begin{align*}
        [
   \begin{array}{l}
    x_{t+1} \\
    y_{t+1}
  \end{array}
] = [
   \begin{array}{l}
    x_{t} \\
    y_{t}
  \end{array}
] - \eta [
   \begin{array}{l}
   \mathbf{I}\quad -\lambda \mathbf{H}_{xy} \\
    \lambda \mathbf{H}_{yx}\quad  \mathbf{I}
  \end{array}
] [
   \begin{array}{l}
    \nabla_x f(x_t, y_t) \\
    -\nabla_y f(x_t, y_t)
  \end{array}
]
    \end{align*}
                    \item[(v)] Concensus Optimization (CO) ~\citep{mescheder2018training}
    \begin{align*}
        [
   \begin{array}{l}
    x_{t+1} \\
    y_{t+1}
  \end{array}
] = [
   \begin{array}{l}
    x_{t} \\
    y_{t}
  \end{array}
] - \eta [
   \begin{array}{l}
    \nabla_x f(x_t, y_t) \\
    -\nabla_y f(x_t, y_t)
  \end{array}
] - \gamma\eta\nabla||\nabla f(x_t,y_t)||^2
    \end{align*}
                        \item[(vi)] Unrolled SGDA ~\citep{metz2016unrolled}
    \begin{align*}
        [
   \begin{array}{l}
    x_{t+1} \\
    y_{t+1}
  \end{array}
] = [
   \begin{array}{l}
    x_{t} \\
    y_{t}
  \end{array}
] - \eta [
   \begin{array}{l}
    (\nabla_x f(x_t, y_k) + \nabla_x y_k)\\
    -(\nabla_y f(x_k, y_t) + \nabla_y x_k)
  \end{array}
] 
    \end{align*}
                            \item[(vii)] Follow-the-Ridge (FR) ~\citep{wang2019solving}
    \begin{align*}
        [
   \begin{array}{l}
    x_{t+1} \\
    y_{t+1}
  \end{array}
] = [
   \begin{array}{l}
    x_{t} \\
    y_{t}
  \end{array}
] - \eta_x [
   \begin{array}{l}
    \mathbf{I}\\
    -\mathbf{H_{yy}^{-1}\mathbf{H}_{yx}} \quad c\mathbf{I}
  \end{array}
] [
   \begin{array}{l}
    \nabla_x f(x_t, y_t)\\
    -\nabla_y f(x_t, y_t)
  \end{array}
], c=\frac{\eta_y}{\eta_x}
    \end{align*}
\end{itemize}
The learning rate for all algorithms is set to of $\eta = 0.05$. In addition, $ \lambda = 1.0$ for SGA and $\gamma = 0.1$ for CO.

\subsubsection{Analysing the updates}
\label{app:updates_toy}

We would like to show that the practical DG algorithm enjoys the same benefits as the theoretical DG. In particular, we would like to show that (i) DG converges to stable fixed points that are desirable and (ii) DG does not converge to unstable fixed points in comparison to SGDA and related methods. \\
To this end we analyse the updates for the toy problems in Fig. 4:
\begin{align*}
    f_1 = -3x^2 - y^2 + 4xy \\
    f_2 = 3x^2 + y^2 + 4xy,
\end{align*}
where x is the $\min$ player and y is the $\max$ player. \\

The two functions are two-dimensional quadratic problems, which are arguably the simplest nontrivial problems. In $f_1$, (0, 0) is a local (and in fact global) minimax; yet we see experimentally that only DG, FR, SGA and CO converge to it; all other method diverge (the trajectory of OGDA and EG almost overlaps). We show this also through analytical analysis by examining the eigenvalues of the Jacobian of the dynamical system.
\paragraph{Updates.}
The GDA updates for the game are:
\begin{align*}
        [
   \begin{array}{l}
    x_{t+1} \\
    y_{t+1}
  \end{array}
] = [
   \begin{array}{l}
    x_{t} \\
    y_{t}
  \end{array}
] - \eta [
   \begin{array}{l}
    \nabla_x f(x_t, y_t) \\
    -\nabla_y f(x_t, y_t)
  \end{array}
] = \\
[
   \begin{array}{l}
    x_{t} \\
    y_{t}
  \end{array}
] - [
   \begin{array}{l}
    \eta (-6x_{t} + 4y_{t}) \\
    -\eta (4x_{t} - 2y_{t})
  \end{array}
]
    \end{align*}

Concretely, 
\begin{align*}
        [
   \begin{array}{l}
    x_{t+1} \\
    y_{t+1}
  \end{array}
] = [
   \begin{array}{l}
    6\eta + 1 \quad -4\eta \\
    4\eta \quad \quad -2\eta + 1
  \end{array}
] [
   \begin{array}{l}
    x_t \\
     y_t
  \end{array}
].
\end{align*}
The eigenvalues of the system are of the form $\lambda=2\eta + 1$, which is greater than 1 since the learning rate $\eta$ is positive. Thus GDA never converges to the desired solution. The main reason behind the divergence of GDA is that the gradient of the leader pushes the system away from the local minimax when it is a local maximum for the leader (as explained in \cite{wang2019solving}). \\
Next we analyse the convergence of DG for the corresponding function. Concretely,
\begin{align*}
    dg_{f_1} = -3x^2 - y_w^2 + 4xy_w - (-3x_w^2 - y^2 + 4x_wy),\text{ where } \\
    x_w = x - \eta(-6x + 4y) \\
    y_w = y + \eta(4x - 2y),
\end{align*}
where both players are \textbf{minimizing} the same objective, i.e. there is no $\min$ and $\max$ players. \\
The Hessian is positive definite at the solution, hence (0, 0) is a local min.
The DG updates are:
\begin{align*}
        [
   \begin{array}{l}
    x_{t+1} \\
    y_{t+1}
  \end{array}
] = [
   \begin{array}{l}
    x_{t} \\
    y_{t}
  \end{array}
] - \eta [
   \begin{array}{l}
    \nabla_x dg_{f_1}(x_t, y_t) \\
    \nabla_y dg_{f_1}(x_t, y_t)
  \end{array}
] = \\
[
   \begin{array}{l}
    x_{t} \\
    y_{t}
  \end{array}
] - \eta [
   \begin{array}{l}
    8\eta((23\eta + 13)x - 8(2\eta y + y)) \\
    8\eta((11\eta + 5)y - 8(2\eta + 1)x)
  \end{array}
]
\end{align*}
Concretely, 
\begin{align*}
        [
   \begin{array}{l}
    x_{t+1} \\
    y_{t+1}
  \end{array}
] = [
   \begin{array}{l}
    1 - 8\eta^2(23\eta + 13) \quad 64\eta^2(2\eta+1) \\
    64\eta^2(2\eta + 1) \quad 1 - 8\eta^2(11\eta+5)
  \end{array}
] [
   \begin{array}{l}
    x_t \\
     y_t
  \end{array}
].
\end{align*}
For the above dynamical system, the eigenvalues of the Jacobian are smaller than 1, for all values of $0<\eta \leq 1$ and the system converges to the local and global minimum which is the desired solution.

\paragraph{Intuition.} Intuitively, the reason why the GDA dynamics exhibit undesirable convergence in most settings can be understood from the gradient updates of the max player. The steps taken by the max player push the dynamics away from the ridge (the region where the gradients with respect to the max player are zero~\cite{wang2019solving}). We show in Appendix~\ref{app:bilinear_games} the DG updates include a correction term that cancels the negative effects due to the max player's update. In particular, Fig.~\ref{fig:intuition} shows that the DG updates are driven towards the ridge and once there, it moves along the ridge until convergence to the solution. FR ~\citep{wang2019solving} is the only other algorithm that seems to behave similarly. However, FR relies on second-order information and currently exists only for the full-batch version.

\begin{wrapfigure}{r}{4.7cm}
  %\vspace{-1.6cm}
  \begin{center}
    \includegraphics[width=0.6\linewidth]{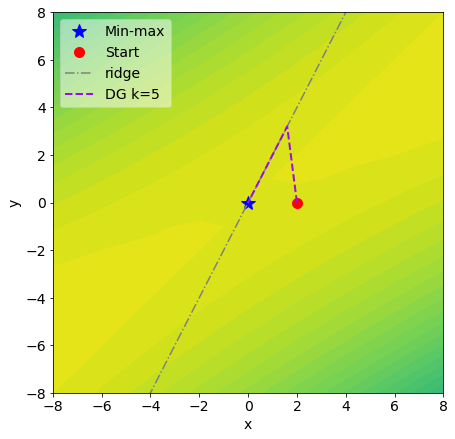}
  \end{center}
  \caption{DG first moves towards the ridge and then along it.}
    \label{fig:intuition}
  \vspace{-0.4cm}
\end{wrapfigure}

In $f_2$, where (0, 0) is not
a local minimax, all algorithms except for DG and FR converge to this undesired stationary point. In this case, the leader is still to blame for the undesirable convergence of GDA (and other variants) since it gets trapped by the gradient pointing to the origin. 

\paragraph{Analysing the updates.}
The GDA updates for the game are:
\begin{align*}
        [
   \begin{array}{l}
    x_{t+1} \\
    y_{t+1}
  \end{array}
] = [
   \begin{array}{l}
    x_{t} \\
    y_{t}
  \end{array}
] - \eta [
   \begin{array}{l}
    \nabla_x f(x_t, y_t) \\
    -\nabla_y f(x_t, y_t)
  \end{array}
] = \\
[
   \begin{array}{l}
    x_{t} \\
    y_{t}
  \end{array}
] - [
   \begin{array}{l}
    \eta (6x_{t} + 4y_{t}) \\
    -\eta (4x_{t} + 2y_{t})
  \end{array}
]
    \end{align*}
Concretely, 
\begin{align*}
        [
   \begin{array}{l}
    x_{t+1} \\
    y_{t+1}
  \end{array}
] = [
   \begin{array}{l}
    1 - 6\eta   \quad -4\eta \\
    4\eta \quad \quad 2\eta + 1
  \end{array}
] [
   \begin{array}{l}
    x_t \\
     y_t
  \end{array}
].
\end{align*}
The eigenvalues of the system are of the form $\lambda=1 - 2\eta$, which is smaller than 1 since the learning rate $\eta$ is positive. Therefore, the GDA update converge to the undesirable point.\\
Now let's see what happens when optimizing DG instead:

\begin{align*}
    dg_{f_2} = 3x^2 + y_w^2 + 4xy_w - (3x_w^2 + y^2 + 4x_wy),\text{ where } \\
    x_w = x - \eta(6x + 4y) \\
    y_w = y + \eta(4x + 2y),
\end{align*}
where both players are \textbf{minimizing} the same objective, i.e. there is no $\min$ and $\max$ players. \\
The Hessian is neither positive definite nor negative definite at the (0, 0). In fact, (0, 0) is a saddle point for the DG objective.
The DG updates are:
\begin{align*}
        [
   \begin{array}{l}
    x_{t+1} \\
    y_{t+1}
  \end{array}
] = [
   \begin{array}{l}
    x_{t} \\
    y_{t}
  \end{array}
] - \eta [
   \begin{array}{l}
    \nabla_x dg_{f_2}(x_t, y_t) \\
    \nabla_y dg_{f_2}(x_t, y_t)
  \end{array}
] = \\
[
   \begin{array}{l}
    x_{t} \\
    y_{t}
  \end{array}
] - \eta [
   \begin{array}{l}
    -8\eta((23\eta - 13)x + 8(2\eta y - y)) \\
    -8\eta((11\eta - 5)y + 8(2\eta - 1)x)
  \end{array}
]
\end{align*}
Concretely, 
\begin{align*}
        [
   \begin{array}{l}
    x_{t+1} \\
    y_{t+1}
  \end{array}
] = [
   \begin{array}{l}
    1 + 8\eta^2(23\eta - 13) \quad  64\eta^2(2\eta-1) \\
    64\eta^2(2\eta - 1) \quad 1 + 8\eta^2(11\eta - 5)
  \end{array}
] [
   \begin{array}{l}
    x_t \\
     y_t
  \end{array}
].
\end{align*}
For the above dynamical system, the eigenvalues of the Jacobian are larger than 1, for all values of $\eta>\frac{1}{3}$.

\subsection{Mixture of Gaussians}
\paragraph{Dataset.} We follow the setup proposed in~\cite{wang2019solving} as they show it is particularly difficult for most standard training methods. The dataset is a mixture of Gaussians composed of 5000 points sampled independently from a distribution $p_D(x) = \frac{1}{3} \mathcal{N} (-4, 0.01) + \frac{1}{3} \mathcal{N} (0, 0.01) + \frac{1}{3} \mathcal{N} (4, 0.01)$, where $\mathcal{N}{\mu, \sigma^2}$ is a one-dimensional Gaussian distribution with mean $\mu$ and variance $\sigma$. Note that our mixture has three 1-dimensional Gaussians centered at -4, 0 and 4. The setup we consider is a full-batch setting, so we sample and re-use the same 5000 points at each iteration. 

\paragraph{Architecture.} In terms of architecture, the generator and discriminator are two hidden layer neural networks with 64 hidden units and tanh activations. The noise vector, $z$, is a 16-dimensional vector sampled from a standard Gaussian distribution.

\paragraph{Hyperparameters.} For all methods, the learning rate for both the generator and discriminator is set to 0.0002. For consensus optimization (CO), 
the gradient penalty coefficient is tuned using grid search over {0.01, 0.03, 0.1, 0.3, 1.0, 3.0, 10.0}. For DG we choose k in {5, 10, 25, 50}.

\subsection{Optimization landscape}
\paragraph{Dataset.} We use the MNIST training set that consists of 50K images, all scaled to be in the range [-1, 1].

\paragraph{Architecture.} We follow the exact setup proposed in~\cite{berard2019closer} and use the DCGAN architecture \cite{radford2015unsupervised} for all objectives (GAN (with the non-saturating loss), WGAN, WGAN-GP, GAN ExtraAdam etc.). The batch-norm layers are replaced with spectral-norm layers~\citep{miyato2018spectral}. 

\paragraph{Hyperaparemeters.} The learning rates for all GAN-based models are $2*10^{-4}$ and $5*10^{-5}$ for G and D, respectively, with $\beta_1=0.9$ where applicable. For the WGAN-based models, the learning rates are set to $8.6*10^{-5}$, with $\beta_1=0.5$ and 5 critic updates in the inner optimization. For WGAN-GP, $\lambda=10$ for the gradient penalty. The batch size for all models is 100. We report the best Inception Score (IS) for 200000 training iterations and report average over 5 rounds. IS is computed using a LeNet classifier pretrained on MNIST with average IS score of 9.9 on real MNIST data.

\paragraph{Additional results.}
In addition to the improved IS scores, we generally find the training with DG to be more stable and convergent (Fig.~\ref{fig:gans_vs_dggans}). We follow the setup from~\cite{berard2019closer} and in Fig.\ref{fig:game_dynamics} further report (i) the eigenvalues of the game Jacobian: As shown in~\cite{berard2019closer} when some eigenvalues have large imaginary parts, then the game has a strong rotational behavior, (ii) the top-k eigenvalues of the diagonal blocks of the Jacobian, which correspond to the Hessian of each player and are informative as to whether the game has reached a (local) NE, (iii) path-norm i.e. the norm of the vector field, which should be low around a stationary point and (iv) path-angle: the angle between the vector field and the linear path from the final to the initial point, which is known to show a "bump" when there is a rotational component to the game. (i)-(iv) combined give more insight into the optimization landscape - in particular as to whether the game is close to a stationary point, whether the point is a local NE and whether there are (strong) rotational dynamics. As shown in Fig.\ref{fig:game_dynamics}, when training a GAN, there is a non-zero rotational component. This can be observed by both the large imaginary part in the eigenvalues of the Jacobian and the path-angle exhibiting a bump, which is not present when training a DG GAN. Moreover, the generator of GANs consistently never reaches a local minimum but instead finds a saddle point. In contrast, the eigenvalues of the Hessians of both players when minimizing DG have no negative values, indicating the convergence to a local NE. Overall, both DG GAN and DG WGAN exhibit a more stable behavior, which is also reflected by improved IS scores.

\paragraph{Additional baselines.} Table~\ref{tab:is_score_app} shows Inception Scores (IS) for various models, averaged over 5 runs with different seeds. We compare GAN and WGAN-based models with different variations. DG gives better scores for both (GAN and WGAN) objectives. We find that DG when using the ALI architecture (which is consistent with our theory) gives additional improvement over DG without the ALI architecture. Finally, we also include results for a GAN trained by performing $k$ optimization steps in the inner (max) or outer loop (min), or both. Overall, DG consistently outperforms the alternatives.

\paragraph{Correlation of IS and DG.}
In Fig.~\ref{fig:app_is_wgans} and ~\ref{fig:app_dg_wgans} we plot the progression of the Inception Score and Duality Gap over epochs. DG shows higher IS and lower DG throughout the entire training. 

\begin{table}[]
    \centering

\begin{tabular}{@{}lc|cc@{}}
\toprule
\multicolumn{2}{c}{\textbf{GAN based}}                 & \multicolumn{2}{c}{\textbf{WGAN based}}       \\ \midrule
\multicolumn{1}{c}{\textbf{Model}} & \textbf{IS}       & \textbf{Model}              & \textbf{IS}      \\
GAN (Adam)                         & 8.58 $\pm$ 0.006  & \multicolumn{1}{l}{WGAN}    & 7.41 $\pm$ 0.029 \\
GAN (ExtraAdam)                    & 8.80 $\pm$ 0.021  &  
\multicolumn{1}{l}{WGAN (ExtraAdam)} & 7.61 $\pm$ 0.014 \\
GAN (SGD)                          & 8.19 $\pm$ 0.017  &                             &                  \\
GAN (RMSProp)                      & 8.69 $\pm$ 0.013  &                             &  \\                 
BiGAN/ALI                          & 8.65 $\pm$ 0.081  &
\multicolumn{1}{l}{WGAN GP} & 9.28 $\pm$ 0.004 \\

\bottomrule
DG (k=5)                                & 9.21 $\pm$ 0.011 & \multicolumn{1}{l}{DG WGAN (k=5)} & 8.31 $\pm$ 0.101 \\
DG   (k=10)                              & \textbf{9.26 $\pm$ 0.021} & \multicolumn{1}{l}{DG WGAN (k=10)} & 9.41 $\pm$ 0.112 \\
DG   (k=25)                              & 9.24 $\pm$ 0.015 & \multicolumn{1}{l}{DG WGAN (k=25)} & \textbf{9.59 $\pm$ 0.014} \\
\bottomrule
DG (without ALI) (k=5)                    & 9.06 $\pm$ 0.006  &\\
DG (without ALI) (k=10)                     & 9.02 $\pm$ 0.013  &\\
DG (without ALI) (k=25)            & 9.01 $\pm$ 0.011 \\
\bottomrule
GAN (Adam) (d=10)                     & 6.92 $\pm$ 0.001  &\\
GAN (Adam) (g=10)                     & 1.015 $\pm$ 0.010  &\\
GAN (Adam) (d=10 and g=10)            & 8.90 $\pm$ 0.012 \\
\bottomrule
\end{tabular} 
      \captionof{table}{\label{tab:is_score_app} Inception score (IS) for various models averaged over 5 runs with different seeds. The average IS score of real MNIST data is 9.9. \\
      The first row is various GAN and WGAN based methods. The second and third rows give results for DG with different number of k optimization steps (the third row is applying the DG objective without the ALI setting). The last row corresponds to training a GAN when the maximization step (d), minimization step (g), or both (d and g) are performed k times.}
\end{table}

\begin{figure}[!htb]
    \centering
    \begin{minipage}{.45\textwidth}
        \centering
        \includegraphics[width=1\linewidth, height=0.2\textheight]{ICML_2021/figures/appendix/Appendix_IS_WGANs.pdf}
        \caption{Inception Score (IS) over epochs.}
        \label{fig:app_is_wgans}
    \end{minipage}%
    \begin{minipage}{0.45\textwidth}
        \centering
        \includegraphics[width=1\linewidth, height=0.2\textheight]{ICML_2021/figures/appendix/Appendix_DG_WGANs.pdf}
        \caption{Duality Gap (DG) over epochs.}
        \label{fig:app_dg_wgans}
    \end{minipage}
\end{figure}

\begin{figure}[htbp]
\captionsetup[subfigure]{labelformat=empty}
\centering
\raisebox{1.5cm}{\rotatebox[origin=c]{90}{GAN}}\quad
\subfloat[]{\includegraphics[width=0.23\textwidth,valign=c]{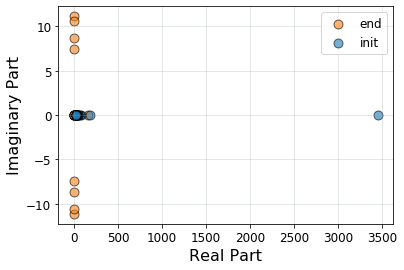}}\hfill
\subfloat[]{\includegraphics[width=0.23\textwidth,valign=c]{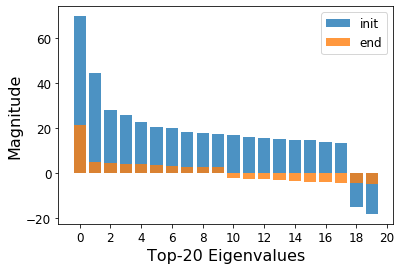}}\hfill
\subfloat[]{\includegraphics[width=0.23\textwidth,valign=c]{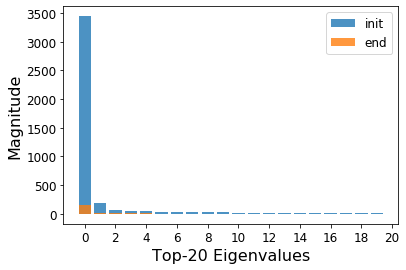}}
\subfloat[]{\includegraphics[width=0.23\textwidth,valign=c]{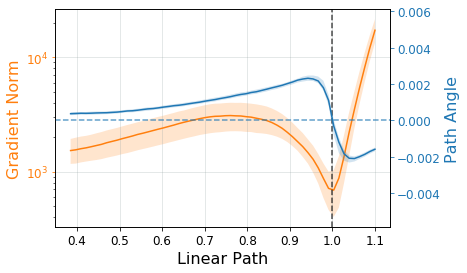}}

\raisebox{1.5cm}{\rotatebox[origin=c]{90}{DG}}\quad
\subfloat[]{\includegraphics[width=0.23\textwidth,valign=c]{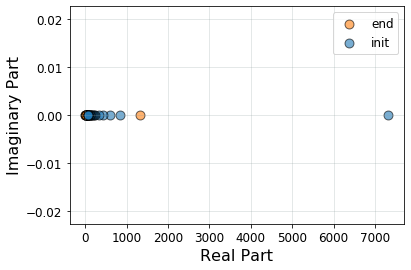}}\hfill
\subfloat[]{\includegraphics[width=0.23\textwidth,valign=c]{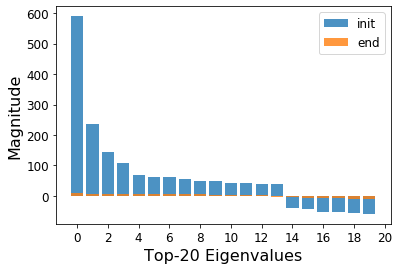}}\hfill
\subfloat[]{\includegraphics[width=0.23\textwidth,valign=c]{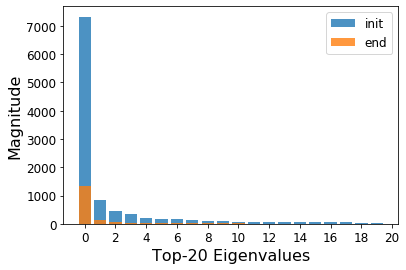}}
\subfloat[]{\includegraphics[width=0.23\textwidth,valign=c]{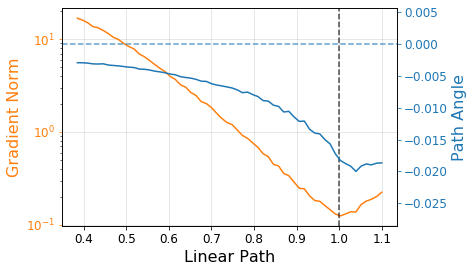}}

\caption{\label{fig:game_dynamics} Optimization landscape for GAN and DG GAN. First column: Eigenvalues of the game at the end point. Second and third column: Top k-Eigenvalues of the Hessian of each player (first G, then D) in
descending order. Fourth column: Path-angles between a random initialization (initial point) and the set of parameters during training achieving the best performance (end
points). }
\end{figure}

\begin{table}[]
\centering
\begin{tabular}{lcc}
\toprule
\textbf{k}         & \textbf{Duality Gap} & \textbf{FID} \\ \hline
\textbf{1}  & (diverges) & (diverges)\\
\textbf{3}  & 1.74 & 23.02\\
\textbf{5} & 0.37 & 22.46 \\
\textbf{10} & 0.24 & 22.14 \\
\bottomrule
\end{tabular}
\caption{\label{tab:fiddifferentk} Duality Gap (DG) and FID for different optimization steps $k$ at the end of training.}
\end{table}

\subsection{Experiments with BigBiGAN}
\paragraph{Dataset.} We use the training sets for MNIST, Fashion-MNIST and CIFAR10 in a conditional and unconditional setting. All datasets have 10 classes. The input images for all datasets are scaled to 32x32 dimensions and are normalized to be within the [0, 1] range.

\paragraph{Architecture.} We use a low-resolution (32 x 32) implementation of BigBiGAN ~\cite{donahue2019large}. The generator, G, and the discriminator unit F have 3 residual blocks each. The discriminator units H and J are 6-layer MLPs with 50 units and skip-connections. The encoder architecture, as suggested in ~\cite{donahue2019large} consists of a higher resolution input (64x64) and a RevNet (Reversible Residual Network) with 13 layers, followed by a 4-layer MLP with units of size 256.

\paragraph{Hyperparameters.} The learning rates for all models are set to $2e-4$ with $\beta_1=0.5$. The batch size is 256 and all models are trained for 50 epochs. We report FID ~\citep{heusel2017gans} for 10,000 generated samples. For DG we set $k=10$ and for the DG metric we report, we used the test sets for each of the datasets respectively. In addition, in Fig.~\ref{fig:bigbigan} we report IS and FID using a pretrained MNIST classifier.

\paragraph{The effect of $k$.} In Tab.~\ref{tab:fiddifferentk} we report the Duality Gap value (DG) and FID at the end of training for different optimization steps $k$. When $k=1$, the training is unstable and eventually diverges. As $k$ increases the training becomes more stable. Overall, we find the optimal value of $k$ to be 10, with no significant difference when further increasing $k$. This is consistent across different experiments (see e.g. Tab.~\ref{tab:is_score_app} and Fig.~\ref{fig:app_is_wgans} and ~\ref{fig:app_dg_wgans}).

Note that DG can be computed for any GAN formulation. When the GAN formulation is the saturating (original) GAN formulation and $k$ goes to infinity, there might be instabilities. To avoid this, one can use the same tools that are used throughout the literature (adding regularizers, batch norms, gradient penalties or simply optimizing for fewer steps)~\citep{schafer2019implicit}. For more thorough analysis of the effect of k on the DG value we refer the reader to~\cite{grnarova2019domain}.

\paragraph{Additional benefits of the minimization setting.} In Section~\ref{sec:bigbigan} we showed that by going away from the adversarial min-max setting, to a minimization setting there are several benefits that inherently appear, including greater stability and  robustness to the choice of hyperparameters. Moreover, we now have access to two informative curves, similar to the standard minimization setting we are used to: DG (train) and DG (validation). In Fig.~\ref{fig:app_dg_train} and Fig.~\ref{fig:app_dg_val} we plot the two curves throughout the training by using the training and validation datasets for the computation of DG, respectively. We do not find any significant difference between the two curves, which is aligned with the observation that GANs suffer from underfitting, rather than overfitting~\citep{adlam2019investigating}. In the future, it would be interesting to further analyze the discrepancy between the curves for different models.

\begin{figure}[!htb]
    \centering
    \begin{minipage}{.45\textwidth}
        \centering
        \includegraphics[width=1\linewidth, height=0.2\textheight]{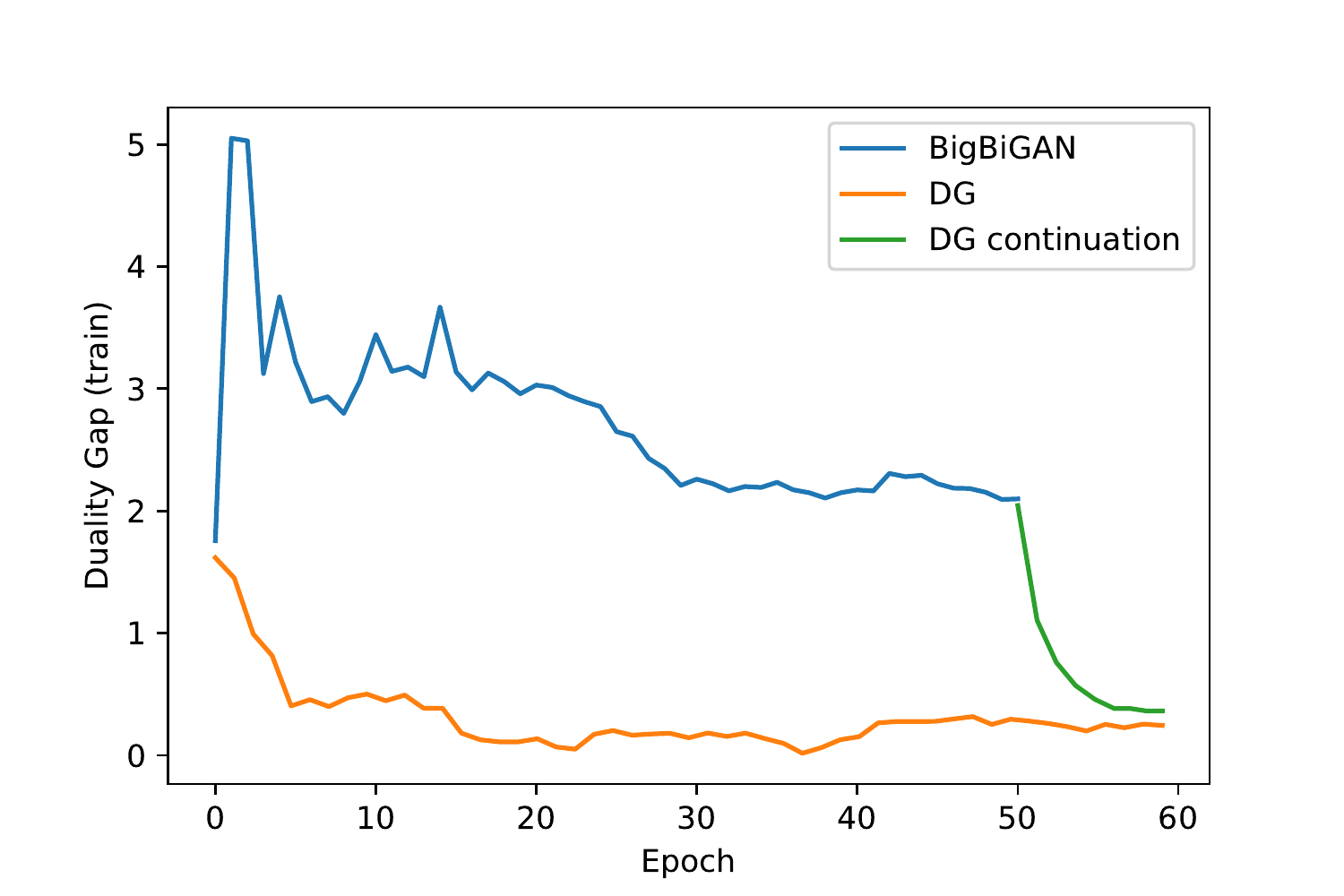}
        \caption{DG (train) over epochs.}
        \label{fig:app_dg_train}
    \end{minipage}%
    \begin{minipage}{0.45\textwidth}
        \centering
        \includegraphics[width=1\linewidth, height=0.2\textheight]{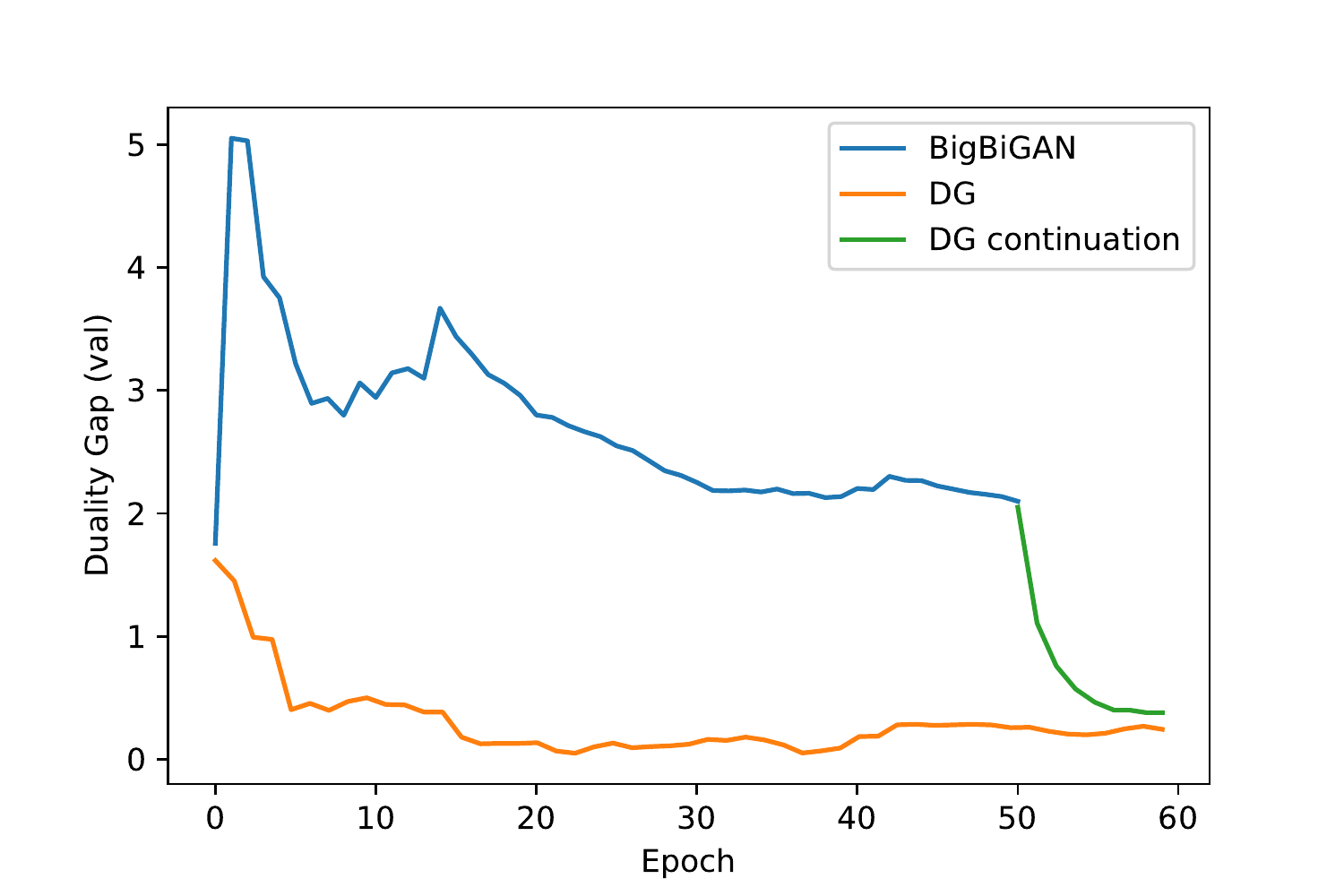}
        \caption{DG (validation) over epochs.}
        \label{fig:app_dg_val}
    \end{minipage}
\end{figure}